\documentclass{article}
\usepackage{iclr2027_conference,times}
\usepackage{amsmath,amssymb,amsthm}
\usepackage{graphicx}
\usepackage{booktabs}
\usepackage{multirow}
\usepackage{xcolor}
\usepackage[hidelinks]{hyperref}
\usepackage{url}
\usepackage{fancyhdr}

\newcommand{\thetaz}{\theta_0}
\newcommand{\forget}{\mathcal{D}_f}
\newcommand{\retain}{\mathcal{D}_r}
\newtheorem{definition}{Definition}
\newtheorem{proposition}{Proposition}

\newenvironment{widefigure}[1][t]{\begin{figure}[#1]}{\end{figure}}
\newenvironment{widetable}[1][t]{\begin{table}[#1]}{\end{table}}

\title{Rethinking Demonstration Unlearning in\\
Imitation Learning for Robotics}

\author{Jiazhuo Li\thanks{Equal contribution.\quad
$^{\dagger}$Corresponding author.} \\
University of Michigan \\
\texttt{jiazhuo@umich.edu} \\
\And
Yu Zhang$^{*}$ \\
Tsinghua Shenzhen International\\
Graduate School \\
\And
Yiming Fei$^{*}$ \\
Zhejiang University \\
\AND
Kangkang Dong$^{\dagger}$ \\
Tsinghua Shenzhen International\\
Graduate School \\
\texttt{dongkangkang@sz.tsinghua.edu.cn} \\
\And
Xiaojun Zhu \\
Tsinghua Shenzhen International\\
Graduate School \\
\And
Houde Liu \\
Tsinghua Shenzhen International\\
Graduate School \\
\And
Jinze Tao \\
Wuhan University of Technology
}

\iclrfinalcopy %

\begin{document}
\maketitle
\fancyhead{}\lhead{Preprint}\thispagestyle{fancy}
\begin{abstract}
Imitation learning for robotics depends on human demonstrations, some of
which people may later ask to remove. Retraining without them is the
natural reference, but its cost grows with policy and dataset scale,
motivating cheaper operators that edit a trained policy. Metrics
inherited from machine unlearning, such as forgetting loss or a single
membership attack, do not establish what an edit removed from a policy
acting in closed loop. We therefore introduce a retrain-calibrated audit
that reads demonstration unlearning along two axes: behavior, whether the
edited policy acts like one retrained without the removed demonstrations,
and evidence, whether an auditor can still detect it was trained on them.
The behavior axis measures action divergence to that retrain at matched
states, calibrated by a floor built from independent retrains, so a
policy at the floor is as close to a retrain as retrains are to each
other. The evidence axis applies a per-demonstration membership attack
against a retrain null, reporting both its rank and its absolute
member-loss level, since rank alone accepts operators that inflate member
losses past the null. A conformal test then combines both axes into one
hypothesis of joint retrain consistency, against a fleet of independent
retrains large enough to reject at conventional significance. Across five
preregistered conditions on three real-robot policy classes and two
simulation suites, the axes dissociate in both directions on one
checkpoint, as an edit may repair task behavior while leaving evidence
unchanged, or reduce evidence while moving behavior away from retraining.
On the ACT arm, a redirect edit restores blind-scored robot success to
18 of 20 trials.
\end{abstract}

\section{Introduction}
\label{sec:intro}

Imitation-learned robot policies \citep{pomerleau1989alvinn} are trained on
demonstrations collected from people, and people can withdraw consent.
Regulation increasingly requires removing a trained model's dependence
on withdrawn data without the cost of full retraining; for a \emph{policy}
in a feedback loop it is not even clear what the request means. We ask the
mechanism question: \emph{when a demonstration is ``unlearned'' from a robot
policy, what exactly is gone?}

\begin{widefigure}[t]
\centering
\includegraphics[width=0.79\linewidth]{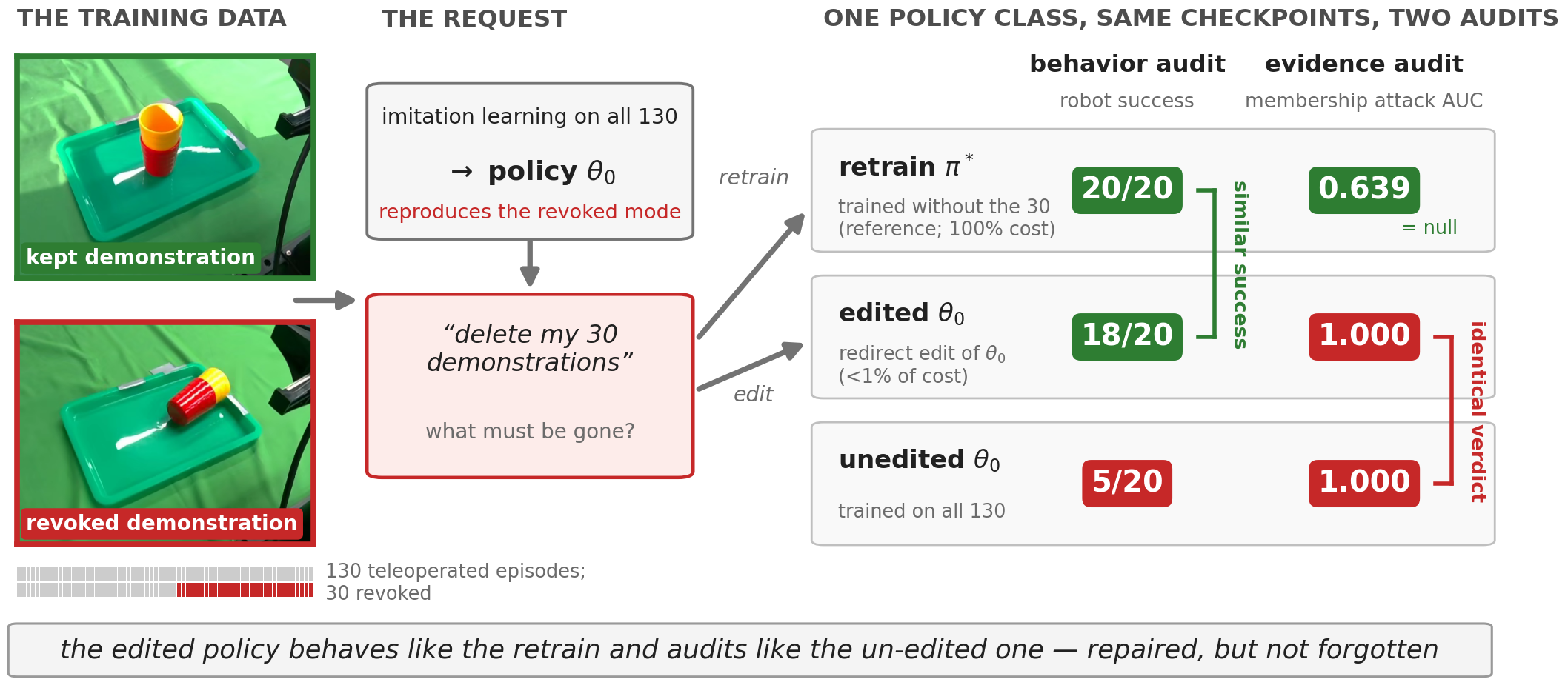}
\caption{\textbf{Demonstration unlearning in imitation learning, and why one
audit is not enough.} \emph{Left}: real frames from the shared teleoperated
dataset --- a kept demonstration seats the cup; the 30 later-revoked ones
teach release at a displaced point that topples it. \emph{Middle}: the policy
trained on all 130 absorbs the revoked mode; the request admits retraining
without the 30 (the reference $\pi^*$) or editing the weights (shown:
\emph{redirect}; the operator ladder and budget-matched FT control are in
Sec.~\ref{sec:five}). \emph{Right}: every number is ACT, and each row's two
verdicts score the \emph{same checkpoint}: similar observed task success to
the retrain's (blind-scored; Table~\ref{tab:trials}), membership-attack AUC identical to the un-edited policy's. The
grouping switches between columns --- repaired, but not forgotten
(Sec.~\ref{sec:dissoc}).}
\label{fig:workflow}
\end{widefigure}

We first quantify how badly the natural answers disagree, on one policy
class, one contamination set, one audit suite (ACT, chunk-50, 30 of 130
demonstrations): a redirect edit restores 18/20 task success on the robot
(retrain: 20/20) and closes a quarter of the executed-slice action gap to
the retrain floor --- and its membership-attack AUC is 1.000,
indistinguishable from the un-edited policy (retrain null 0.639).
Gradient ascent runs it backwards on that lineage --- rank audit through
the null, behavior moving \emph{farther} from the retrain than where it
started --- yet on two further independently trained $\thetaz$ seeds the
same budgets leave the rank audit still flagging (Sec.~\ref{sec:dissoc}).
Same audits, same operators, opposite verdicts.
Table~\ref{tab:ledger} (appendix) collects twelve such measured
contradictions.

The root cause is that ``deleted'' conflates two axes. \textbf{Behavior}: the
policy's conduct approaches what a policy retrained without the
forget set would do. \textbf{Evidence}: an auditor with query access can no
longer distinguish the edited policy from that retrain. Different instruments,
different operators, and --- our central result --- moved independently in
both directions by operators in current use (Table~\ref{tab:ledger}).
The claim this paper defends is structural:
\textbf{retrain-equivalent deletion is a conjunction, not a one-axis score.}
Behavior can move far toward the retrain while membership evidence stays
unchanged; a rank audit can accept checkpoints that absolute evidence and
counterfactual behavior reject; no edit we audit is jointly
retrain-consistent. Any one-axis evaluation can return an erroneous
deletion verdict.

\textbf{Contributions.}
\begin{itemize}
\item \textbf{A retrain-calibrated joint criterion and its refutation-only
reading}: behavior against a \emph{floor} of independent retrains, evidence
against a retrain
\emph{null} in rank \emph{and} absolute form --- rank alone is provably
overshoot-blind (Proposition~\ref{prop:rank}); a deletion claim requires
both; the audit refutes claims but cannot certify them, and its conformal
instantiation against a $19$-replica retrain fleet rejects every audited
ACT checkpoint at $p{=}0.05$. The paired hardware protocol, audit suite,
and ledger (Table~\ref{tab:ledger}) are released.
\item \textbf{A same-checkpoint double dissociation, and an
uncontrolled operator.} On ACT, the redirect that restores 18/20 hardware
success and closes 24.9\% of the executed-slice offline gap carries
evidence identical to the un-edited policy (AUC 1.000, null 0.639) ---
replicated on three independently trained $\thetaz$ seeds --- while
ascent's evidence endpoint is set by its self-referential stopping rule,
not its budget: the same ladder lands at \texttt{mem/null}
$1.83/0.91/1.19$ across those seeds.
\item \textbf{Cross-class evidence that standard verdicts each admit false
deletions.} Task success ($\pi_{0.5}$ FT: 18/20 on hardware at rank AUC
1.000, offline conduct moved 4.6\%),
forgetting loss (DP: $32\times$ loss-gap closure at zero action movement),
and rank acceptance (three classes) each pass checkpoints the joint criterion
rejects; a prospective PushT pre-registration is refuted at its frozen dose,
and operator \emph{applicability} proves a measurable property of the
request itself (3.65\% editable intersection vs.\ a pre-set 5\% bar).
\end{itemize}

\section{Related Work}
\label{sec:related}

\textbf{Machine unlearning and its audits.} Exact and certified unlearning
define deletion against retraining
\citep{bourtoule2021sisa,guo2020certified,ginart2019making,sekhari2021remember};
the LLM line optimizes forgetting objectives at scale
\citep{eldan2023whos,maini2024tofu,zhang2024npo,li2024wmdp}; membership inference is the
standard evidence instrument \citep{shokri2017membership,carlini2022lira};
graph unlearning already carries the same suspicion of surface-level verdicts
\citep{zhang2026inversion,zhang2026attackby,zhang2024partial}. Our
contribution to this line is the calibration discipline for \emph{policies}
--- both axes read against retraining, floors for behavior and nulls for
evidence, rather than against zero --- and the demonstration that the two
calibrated axes dissociate.

\textbf{Editing and steering robot policies.} Behavior Uncloning
\citep{more2026}, concurrent, distills mode redirection into policy weights
with a retain loss (its MoRE operator), scored by success against a
filtered-data retrain on DP and $\pi_{0.5}$ hardware --- the operator
closest to our redirect, no deletion semantics, no evidence axis: precisely
the conflation we audit. RedFlow \citep{redflow2026} shares the word, not
the goal: it converts failure experience into corrective supervision for
flow-matching VLAs, with no forget set and no auditor.
VLA-Forget \citep{vlaforget2026} is the closest unlearning work, running
GA/NPO-family operators on OpenVLA-scale backbones
\citep{kim2024openvla} at 30\% forget --- on a
\emph{language-prompted} PushT variant, so its setting is not directly
comparable to our plain diffusion-PushT arm. In RL, TrajDeleter
\citep{trajdeleter2024} forgets trajectories through the critic, and
concurrent non-archival work applies Fisher-weighted forgetting to diffusion
policies \citep{rff2025}; from-scratch BC has no critic, so the lineage does not
transfer.
None pairs demonstration-identity deletion with a hardware evaluation and a
retrain-calibrated evidence audit; that joint slot is the one this paper
occupies.

\textbf{Attribution and data curation for IL.} Influence-function methods
\citep{koh2017understanding}, CUPID \citep{cupid2025}, and DataMIL
\citep{datamil2025} select or attribute demonstrations for performance; they answer ``which data mattered,'' not ``is it gone.'' Our
Stage-1 identification (loss scorer, $P@60{=}1.000$ on corr60) borrows their
machinery for the complementary problem.

\section{What Deletion Requires: A Retrain-Calibrated Definition}
\label{sec:setup}

\textbf{Demonstration-level unlearning.} A training procedure $\mathcal{T}$
fits $\pi_{\thetaz} = \mathcal{T}(\mathcal{D})$ on
$\mathcal{D} = \retain \cup \forget$; $\mathcal{T}$ may train from
scratch or adapt a frozen pretrained backbone with corpus disjoint from
$\mathcal{D}$ --- we require only that $\mathcal{D}$ is known and re-running
$\mathcal{T}$ affordable. A deletion request names $\forget$ by identity
(episode indices), not by an oracle behavior label; an operator produces
$\pi_\theta$ from $(\pi_{\thetaz}, \forget, \retain)$ at compute
$\ll$ re-running $\mathcal{T}$. The reference object throughout is the
\textbf{retrain counterfactual} $\pi^* = \mathcal{T}(\retain)$ --- from
scratch for ACT and DP; for $\pi_{0.5}$ the same LoRA adapter re-fit on
$\retain$ from the same frozen base. ``Success'' is proximity to $\pi^*$
--- \emph{not} task-success maximization, which a harmful-mode deletion may
even reduce.

\begin{definition}[Behavior axis]
Deployed-path action divergence between $\pi_\theta$ and $\pi^*$ at matched
states, reported per state region (PRE / BRANCH / POST / TAIL of the
contaminated trajectories) in raw action units, and calibrated by the
\textbf{floor}: the same divergence measured between independent retrains
$\pi^*_i, \pi^*_j$. A model at the floor is as close to a retrain as retrains
are to each other. This is an \emph{offline counterfactual-proximity} audit at banked
states; closed-loop success is a separate manifestation channel, never
equated with it (Sec.~\ref{sec:hw}).
\end{definition}

\begin{definition}[Evidence axis]
Membership advantage of an auditor distinguishing $\forget$ members from
matched non-members through $\pi_\theta$, calibrated by the \textbf{retrain
null} (the same attack run on $\pi^*$), reported as (i) rank AUC and (ii) an
absolute member-loss level relative to the null (\texttt{mem/null}). Rank
statistics are scale-free but blind to overshoot; the pair is the audit.
Evidence claims are scoped to the tested loss-based auditor family.
\end{definition}

\subsection{Four outcomes, and why no single audit separates them}
\label{sec:outcomes}
Because the axes are independent an edited policy occupies one of four
quadrants. \textbf{(1) Behavioral repair}: conduct approaches $\pi^*$ while the
auditor still separates members --- the policy stops misbehaving and stays
fully attributable. \textbf{(2) Rank-audit acceptance}: the membership
\emph{ranking} stops distinguishing members, which certifies nothing, since a
rank statistic is invariant to any monotone inflation of both populations.
\textbf{(3) Absolute overshoot}: the member-loss level moves \emph{past} the
null, leaving the policy anomalously bad at data it once fit --- an
anomaly a two-sided auditor reads as evidence of \emph{editing} (inconsistency
with retraining; not, by itself, proof the original memory persists). \textbf{(4) Joint
retrain consistency}: conduct at the floor \emph{and} both tested evidence
statistics consistent with the retrain null at once. Only (4) is eligible for
a deletion claim under this audit --- the minimum condition our protocol
can check, not proof the memory is gone; a rank-only audit reports (2) as
success and cannot tell (2) from (3). Hence rank and absolute levels are reported as a pair throughout, and the
two-axis matrix of Table~\ref{tab:matrix} is the central claim in table
form: no tested checkpoint reaches (4). The pairing is a fact, not a hedge:
any ROC/AUC verdict is invariant under strictly increasing score
transformations, so ``at the null'' and ``past it'' can be
rank-indistinguishable (Proposition~\ref{prop:rank}, appendix;
Table~\ref{tab:evladder} realizes it on B200 vs.\ B300).

The conjunction is testable as \emph{one} hypothesis, not an AND of
marginal checks. Collect the audit into a vector
$z(\theta) = [\,\text{BRANCH divergence},\ |\log \texttt{mem/null}|,\
|\mathrm{AUC} - \text{null}|\,]$ and ask whether $z(\theta)$ is
exchangeable with the retrain fleet's $\{z(\pi^*_i)\}$: score each
replica's standardized sup-norm distance from the others (leave-one-out),
score the edit against all $K$, and report the conformal
$p = (1 + \#\{\text{replicas at least as atypical}\})/(K{+}1)$. The
test's floor is $1/(K{+}1)$ --- one-sided toward refutation, never
certification --- so we trained the retrain fleet out to $K{=}19$
replicas, putting the floor at $1/20{=}0.05$: refutation at conventional
significance is attainable, certification still is not. Measured
(Appendix~\ref{app:joint}): every audited ACT checkpoint is rejected at
$p{=}0.050$, each more atypical than all $19$ replicas, so every $p$
sits at the fleet's attainable minimum; the effect sizes carry the
margin --- checkpoint nonconformity $16$--$57$ against a replica range
of $0.4$--$3.3$ (median $1.0$), redirect nearest, ascent farthest. The
rank margin agrees ($0.105$ vs.\ minimum resolvable $0.203$); the
asymmetry is structural, and no axis trades against the other.

\section{Experimental Design}
\label{sec:design}
One semantic deletion-request family instantiated across settings ---
mode-level contamination named by episode identity --- five conditions, one
audit: three robot classes, two sim arms. Every result carries one of four
statuses: \emph{pre-registered confirmatory} (P1--P5,
Appendix~\ref{app:prereg}), \emph{amended before execution} (the P2
feasibility gate), \emph{exploratory} (the over-budget PushT rung), or
\emph{invalidated} (the unstable-lr sweep, Appendix~\ref{app:lrcliff}).

\subsection{Deletion requests and contamination}
\label{sec:requests}
\emph{Contamination.} 30 of 130 demonstrations teach a coherent wrong
mode of the same task (transport toward a displaced release point; on
hardware this topples the target object). The mode is \emph{plausible} ---
mode-level, not noise; zero-mean corruptions are a regime where fine-tuning
wins and are excluded by scope.

Two simulation arms complete the matrix. robomimic-BC (Can/Square) uses a
constructed contamination split (corr60) and a natural worse-operator split.
Diffusion-PushT uses \emph{mirror contamination}: 80 of its 286 training
episodes are reflected about the vertical axis, so each mirrored
demonstration pushes the T toward the mirrored goal (Fig.~\ref{fig:pusht}a)
--- a coherent wrong-goal mode, not noise. The reflection is exactly rigid-body-consistent (Kabsch residual
0.0000; Appendix~\ref{app:sim}). $\thetaz$ trains on all 286 episodes (3 seeds),
the ceiling on the 206 clean ones (5 seeds); full protocol in
Appendix~\ref{app:sim}.

\subsection{Five conditions, three policy classes}
\label{sec:five}
Each policy class instantiates the same pre-registered ladder:
\textbf{ceiling} $\pi^*$ (retrain on $\retain$; plus 2--3 extra
retrain seeds forming the floor), \textbf{$\thetaz$} (trained on all of
$\mathcal{D}$), \textbf{redirect} (edit toward retrieved clean counterfactual
actions at branch states; the operator family of MoRE \citep{more2026}),
\textbf{ascent}
(calibrated gradient ascent on $\forget$), and \textbf{FT} (fine-tuning on
$\retain$ at the same budget --- the control that separates ``unlearning''
from ``more clean training'').

Classes: ACT \citep{zhao2023act} (chunked $L_1$ regression,
chunk 50), Diffusion Policy \citep{chi2023diffusion} (DDPM, action horizon
16), and $\pi_{0.5}$ \citep{intelligence2025pi05,black2024pi0}
(3B-parameter flow-matching VLA, LoRA \citep{hu2021lora}); training and
deployment through LeRobot \citep{cadene2024lerobot}, simulation arms on
robomimic \citep{mandlekar2021robomimic}. ACT and DP train from scratch;
$\pi_{0.5}$ is LoRA-adapted from a frozen pretrained base (Sec.~\ref{sec:setup}),
scoping its deletion claim to the adapter's training set. All three use the
same 130-episode cup-transport dataset collected on an AgileX PiPER arm
(dual RealSense, 30 Hz); the same 770 branch-window frames and retrieved
clean targets port across classes unchanged.

\subsection{Offline audits, conventions, and controls}
\label{sec:offline}
\emph{Measurement conventions.} Every action-space number carries its
signature $[\text{units},\ \text{reference},\ \text{region},\ \text{horizon},
\ \text{samples}]$; horizons as measured: $\pi_{0.5}$ its full executed
50-step chunk, DP its executed 8-step slice, and ACT its \emph{executed
25-step slice} --- the primary convention; ACT's full 50-step predicted plan
is reported as a diagnostic, because $\thetaz$'s planned divergence ramps
late in the chunk (Fig.~\ref{fig:manifest}a): the BRANCH gap is $2.2$ raw
units on the executed slice vs.\ $7.9$ full-plan, so redirect's closure
reads 24.9\% (primary) vs.\ 75.1\% (full plan), while operator ordering,
every ascent sign, and every evidence number are convention-independent
(both tables from the same banks; Appendix~\ref{app:floors}). That the
closure fraction is convention-dependent while closed-loop repair is not is
itself a manifestation-blindness exhibit. Full conventions: Appendix~\ref{app:conv}.

\emph{Banks and floors.} For each class we bank sampled action chunks at
matched observations across all conditions (common random numbers, 4 samples
per observation, own-postprocessor raw units), label regions, and score
divergence per region against the ceiling with paired bootstrap CIs. The
ACT c50 floor, from 3 independent retrains (3 pairs), executed-slice raw
joint units:
BRANCH $1.81\pm0.07$, POST $2.01\pm0.08$, PRE $2.83$, TAIL $7.27\pm0.97$
(pooled forget $3.77$; full-plan values: Appendix~\ref{app:floors}); the
DP floor uses 4 retrains (6 pairs)
and the PushT behavioral null uses 5 (below); the ACT retrain seeds nest
--- floor $\subset$ evidence null $\subset$ the $K{=}19$ conformal fleet
(Appendix~\ref{app:joint}). The per-region floor is load-bearing: floors
span $4.0\times$ across regions, and a pooled floor would misread the best
redirect as at the null. A normalized-space replication of the full table
(same ordering, different scale) is in the appendix as the cross-check.

\emph{Membership audits.} Per-demonstration loss-based MIA \citep{yeom2018privacy,salem2019mlleaks} with
bootstrap-CI AUC against the pooled retrain null (all floor seeds join the
null), plus the absolute \texttt{mem/null} column. Per-demonstration
scores average 64 sampled frames, and the frame draw is itself a noise
source: a $5{\times}3$ seed-by-draw decomposition of the ACT null puts the
frame-draw SD ($0.039$) \emph{above} the retrain-seed SD ($0.021$; total
per-reading SD $0.052$), so the null is the grand mean of all 15 readings ($0.639$) and the
margin $2\,$SD $= 0.105$, not the pre-registered seed-only $0.074$.
$\pi_{0.5}$ and PushT cells are all measured. PushT's non-members must be
synthesized (every mirrored episode is in training), so its null doubles as a
validity check on that construction --- see Sec.~\ref{sec:generalize}.

\emph{Guards.} Retain-loss ceiling, held-out contaminated-calibration
guard (catches collateral the retain guard misses), and budget-matched FT
read alongside every edit.

\subsection{Real-robot protocol}
\label{sec:robotproto}
Paired design: a fixed sequence of 20 initial cup positions, identical
across conditions (trial $k$ matches), set by a removable template so no
fiducial appears in any observation; collection is \emph{sequential per
model} on both arms, in five-episode blocks with the identical position
sequence (server-restart cost forbids finer interleaving on $\pi_{0.5}$;
the ACT session ran likewise) --- so
paired contrasts share initial positions, not collection time, and session
drift is a disclosed confound rather than a randomized-out one, consistent
with the session-dependence finding of Sec.~\ref{sec:hw}; episodes
Enter-gated, capped at 25 s; live keys are
provisional --- outcomes are scored blind from video with a failure-mode
taxonomy; continuous outcomes (joint trajectories, 30 Hz, raw degrees) are
the primary quantitative channel: paired, interval-honest small-$n$ hardware
evaluation \citep{tri2025lbm,step2025stopping,chi2024umi}. A planned
physical membership probe was superseded pre-collection by the
release-angle AUC of Sec.~\ref{sec:hw}. The simulation arm's
pre-registered predictions and their outcomes are
verbatim in Appendix~\ref{app:prereg}; the hardware protocol was frozen
2026-08-10 and twice amended pre-collection, and its verbatim text and
amendment log accompany the released artifacts.

\section{Results}
\label{sec:results}

\subsection{The two axes dissociate on the same checkpoints}
\label{sec:dissoc}
Both axes, one class, one contamination set, one audit suite: every contrast
below is the \emph{same checkpoint} scored twice, and Table~\ref{tab:matrix}
(appendix) extends the same two-axis reading to every audited checkpoint
under each arm's own pre-declared instruments --- none places both axes
inside its retrain band.

Redirect moves toward the retrain counterfactual, ascent away: in raw
units against the 3-seed floor on the executed slice,
R200 closes 24.9\% of the $\thetaz$-to-floor BRANCH gap
(75.1\% full-plan --- the late-ramp diagnostic of Sec.~\ref{sec:offline};
the floor is reached under neither convention). R400 is the same within
noise while retain damage grows ---
redirect saturates at its budget. No ascent rung approaches the floor, and
beyond B50 ($+1.5\%$, inside the floor's own noise) the ladder degrades
monotonically ($-14.0\%$ at B100 to $-182\%$ at B400).

Fine-tuning at matched budget does neither.
FT200 moves $+1.6\%$ toward the floor %
and closes $-3.6\%$ of the loss gap %
(DP: 0.19\% at 200, 5.9\% at 500 --- FT ``forgetting'' grows with
budget; subtract it from every edit's claim).

The full rung-by-rung evidence ladder is Table~\ref{tab:evladder}: one
attack, one null pooled over 5 retrain seeds $\times$ 3 frame draws
(grand mean $0.639$, per-reading SD $0.052$, margin $0.105$;
Sec.~\ref{sec:offline}). The elevation above chance is measured to be
benign: scoring the \emph{entire} 50-episode pool and permuting the
member/non-member labels puts every retrain seed mid-distribution among
random balanced regroupings ($p = 0.55$--$0.97$; no-membership band
$[0.24, 0.76]$ at this pool size), with no session-order drift
($|\rho| \le 0.10$) --- small-pool sampling width, not group identity:
exactly what calibrating against the measured null absorbs. The cross-arm control agrees: DP and $\pi_{0.5}$ audit the
\emph{same} 30/10 pools and read $0.534 \pm 0.012$ and $0.558 \pm 0.012$
(Sec.~\ref{sec:generalize}). Against that, $\thetaz$, redirect, and FT
sit at $1.000$ (permutation $p < 10^{-3}$).
\begin{table}[t]
\centering
\caption{\textbf{The ACT c50 evidence ladder (primary $\thetaz$
lineage)}, one attack, one null pooled over 5 retrain seeds $\times$ 3
frame draws (grand mean $0.639$, margin $0.105$; Sec.~\ref{sec:offline}).
Rank is scale-free but overshoot-blind; the absolute \texttt{mem/null}
level is not. No rung of this ladder lands on the null in absolute terms:
every episode-bootstrap CI excludes $1.0$ (B200 $[0.79, 0.89]$; B300
$[1.48, 1.67]$; B400 $[1.72, 1.94]$). The $1.000$ boundary readings
replicate on three independently trained $\thetaz$ seeds; the ascent
endpoint does not (Sec.~\ref{sec:dissoc}).}
\label{tab:evladder}
\scriptsize
\setlength{\tabcolsep}{4pt}
\begin{tabular}{@{}lccc@{}}
\toprule
 & rank AUC [95\% CI] & vs.\ null & \texttt{mem/null} \\
\midrule
$\thetaz$ / FT200 / R200 & $1.000$ [1.00, 1.00] & $+0.36$ & $0.22$ / $0.19$ / $0.40$ \\
R400             & $0.990$ [0.96, 1.00] & $+0.35$ & $0.44$ \\
B50              & $0.997$ [0.98, 1.00] & $+0.36$ & $0.38$ \\
B100             & $0.940$ [0.86, 1.00] & $+0.30$ & $0.57$ \\
B200             & $0.650$ [0.44, 0.84] & $+0.01$ & $0.84$ \\
B300             & $0.557$ [0.34, 0.76] & $-0.08$ & $1.59$ \\
B400             & $0.533$ [0.30, 0.75] & $-0.11$ & $1.83$ \\
\bottomrule
\end{tabular}
\end{table}

Two readings.

(i) The double dissociation is within-campaign, on the primary
$\thetaz$ lineage: the redirect that closed a quarter of the executed
behavioral gap carries loss-attack evidence indistinguishable from the
un-edited policy, while ascent's rank evidence crosses the null between
its 200- and 300-step rungs (interpolated 199, paired bootstrap 95\% CI
$[116,358]$; a quarter of resamples never cross) and its behavior never
crosses at any budget --- no rung \emph{lands} on the null in absolute
terms (every episode-bootstrap CI excludes it,
Table~\ref{tab:evladder}), and a calibrated per-episode second attack
sees the same step-over ($z$: $-3.1$ at B200 to $+10.3$ at B300;
Appendix~\ref{app:sim}).

Seed replication then re-scopes the two directions differently.
Rerunning redirect and the full ascent ladder on two further
independently trained $\thetaz$ seeds (identical recipe and budgets;
all three start at AUC $1.000$): masking replicates exactly --- AUC
$1.000$, \texttt{mem/null} $0.39$--$0.41$ on every lineage --- but the
ascent endpoint does not: identical full budgets land at
\texttt{mem/null} $1.83/0.91/1.19$ (rank $0.533/0.913/0.863$), because
the operator's self-referential stopping rule saturates at
lineage-dependent levels. The training draw, not the budget, sets the
evidence endpoint (mechanism, and the mirror-image rank/absolute
disagreement it produces: Appendix~\ref{app:joint}). Banked conduct
agrees (full-plan bank units, all three lineages in one report):
redirect lands BRANCH at the retrain level where it edits on every
lineage (divergence/null $1.05/1.05/1.03$, TAIL untouched, same retain
cost), while ascent's overshoot is again primary-only: B300 moves
BRANCH $128\%$ of the $\thetaz$ gap away from the retrain there,
$63\%$/$27\%$ \emph{toward} it on the other two. Both legs of masking
replicate; neither leg of ascent does.

(ii) The same checkpoints close the loop on hardware: blind-scored under
the paired protocol (q25 schedule, $4\times5$ episodes per condition),
success orders
ceiling 20/20 $>$ redirect 18/20 $>$ FT 9/20 $>$ ascent 7/20 $>$ $\thetaz$
5/20, matching the offline floor table. Paired exact tests against the
ceiling separate every condition ($p \le 10^{-3}$) \emph{except} redirect
($p = 0.50$, Table~\ref{tab:trials}) --- and that redirect, not separable
from the retrain at $n{=}20$ (failure to detect, not equivalence), is the
checkpoint whose attack AUC is 1.000. ($\thetaz$'s valid 20: the four complete consecutive blocks after
a mid-collection rig fault and 29-minute gap, fixed before scoring.)

\emph{Why the axes come apart} is measured, not asserted. The request's
behavioral payload is \emph{local} (the BRANCH window, $8.6\%$ of member
frames); the membership statistic is \emph{global}, a whole-episode mean.
Scoring every frame by region (Appendix~\ref{app:joint}): redirect reaches
the retrain null exactly where it edits (BRANCH \texttt{mem/null} $1.07$
vs.\ $\thetaz$'s $0.26$) and leaves tail memorization untouched ($0.20$
vs.\ $0.17$) --- the repair is real, local, and averaged away, hence AUC
$1.000$. Ascent (the overshooting lineage), episode-global by
objective, overshoots \emph{every} region ($1.27$--$2.31$); FT moves
none ($\le 0.32$ everywhere). In
parameter space, the same split: gradient cosine $0.30$ at the repaired
R200 vs.\ $0.66$ at $\thetaz$ --- descending one objective barely moves
the other.

\subsection{Task success does not certify deletion}
\label{sec:hw}
\begin{table}[t]
\centering
\caption{\textbf{$\pi_{0.5}$ hardware, all five conditions}, $n{=}20$ each
under the paired protocol. Release angle orders with success; $\thetaz$'s
tight spread is a consistent displaced aim, not noise. Blind-scored labels
(Table~\ref{tab:trials}); continuous channels from recorded joint
trajectories.}
\label{tab:pi05robot}
\scriptsize
\setlength{\tabcolsep}{3.5pt}
\begin{tabular}{@{}lcccc@{}}
\toprule
 & succ. & release ($^\circ$) & grasp ($^\circ$) & carry (s) \\
\midrule
ceiling       & 19/20 & $15.1\pm1.6$ & $-20.6\pm0.8$ & 2.95 \\
redirect (H)  & 17/20 & $14.3\pm1.5$ & $-16.6\pm5.8$ & 2.43 \\
FT (D)        & 18/20 & $13.8\pm2.3$ & $-15.0\pm8.8$ & 2.87 \\
ascent (F)    &  8/20 & $13.3\pm1.8$ & $-15.5\pm7.0$ & 2.33 \\
$\thetaz$ (B) &  2/20 & $11.3\pm1.5$ & $-17.2\pm5.9$ & 1.92 \\
\bottomrule
\end{tabular}
\end{table}
All five $\pi_{0.5}$ conditions ran $n{=}20$ under the paired protocol
(Table~\ref{tab:pi05robot}). Outcomes are blind-scored from
anonymized, shuffled clips; across both arms the live keys agreed with the
blind labels on 199 of 200 episodes (the flip: an ACT ascent episode; blind
authoritative).
Three observations (Fig.~\ref{fig:robot}; blind-scored taxonomy in
Table~\ref{tab:trials}). Release aim is the harm mechanism: success
orders with release angle, and $\thetaz$'s tight $11.3^\circ\!\pm\!1.5$
is a \emph{consistent} displaced aim, not noise.
The 130-lineage grasp dispersion survives every operator ($\pm5.8$--$8.8$
vs.\ ceiling's $\pm0.8$): repair restored the release aim, not the grasp
distribution, and success tracks only the former. Training lineage also
leaks through \emph{conduct}: release angle alone distinguishes $\thetaz$
from ceiling rollouts at AUC $0.958$ (a conduct signal, not membership
inference), and operators suppress it unevenly (ascent $0.775$, FT $0.653$,
redirect $0.634$): redirect is \emph{most} deniable in conduct yet
\emph{least} deniable to the loss attack. FT is the section title's
sharpest instance: it recovers 16 of the 17-point success gap while moving
offline counterfactual conduct $4.6\%$ and rank evidence not at all ---
success is a manifestation channel, not the behavior axis, and the two
dissociate inside a single control.
The blind taxonomy makes the mode's signature explicit: sway-and-knock
accounts for 17 of $\thetaz$'s 18 failures, the terminal-swing rate falls
across the ladder ($\thetaz$ 90\%, ascent 60\%, FT and redirect 0--5\%; on
ACT: 75/50/40/0\%), and redirect's residual
failures carry a different syntax entirely (dropped-short, timeout; zero
sway on either arm) --- the edit removed the contaminated mode, not merely
some failures. Crucially, expression is initial-state-dependent at fixed
weights: the harm is absent exactly at $\thetaz$'s two successful cup
positions --- why closed-loop testing at feasible trial counts cannot
\emph{alone} certify absence of the harmful mode. A pre-registered probe of
whether the replanning schedule toggles expression was refuted (opposite
sign, n.s.\ at $n{=}10$); cross-session rates are not quantifiable in this
campaign (the preceding session was voided by a keying artifact), so we
report state-dependence within the session of record. And plainly:
collection was sequential per model on both arms
(Sec.~\ref{sec:robotproto}); paired positions control initial state, not
session drift, so every hardware contrast here is sequential-collection
evidence. The confirmatory design is no longer unspecified: an interleaved
multi-session protocol (retrain / $\thetaz$ / redirect / ascent, order
randomized within session, prespecified noninferiority margins, power-set
trial counts, two blind raters) is frozen in the released artifacts; until
it runs, the sequential design bounds what this arm can claim.

\begin{widefigure}[t]
\centering
\includegraphics[width=0.55\linewidth]{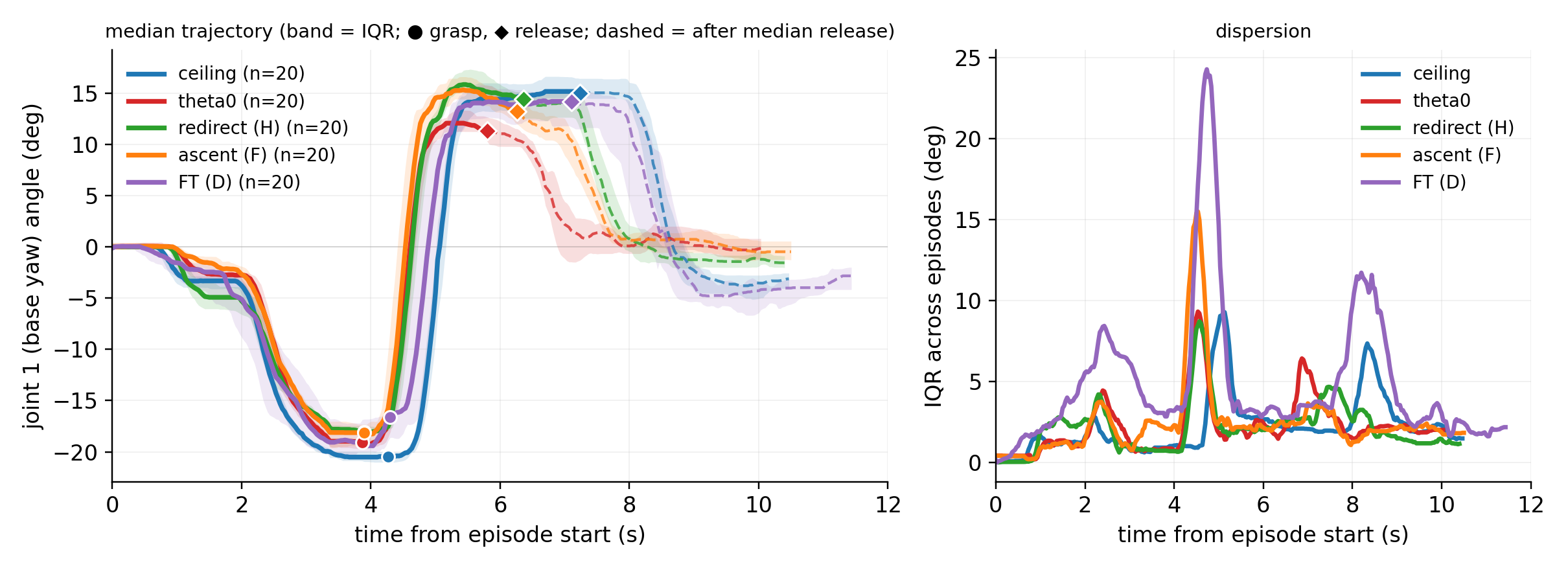}\\[2pt]
\includegraphics[width=0.63\linewidth]{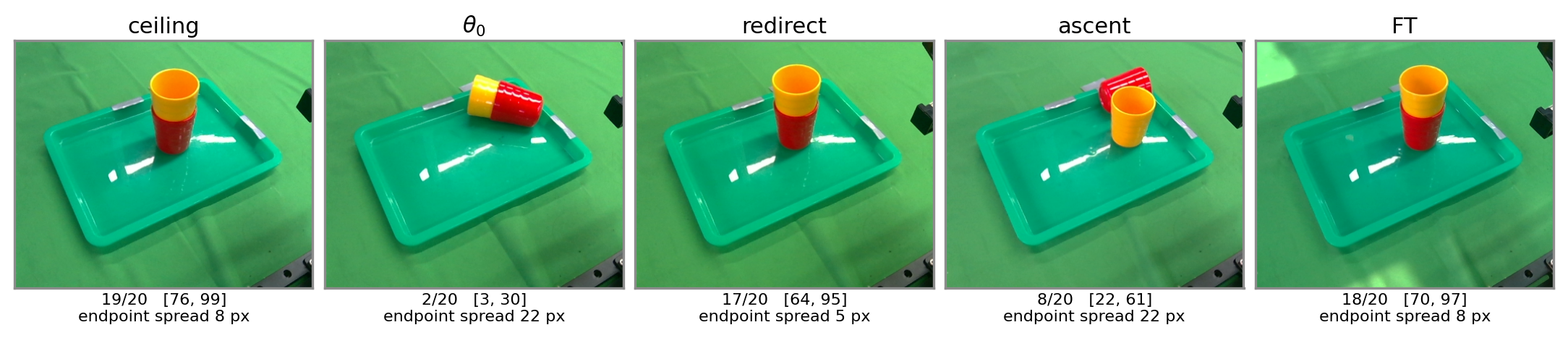}
\caption{\textbf{One mechanism, measured and averaged ($\pi_{0.5}$).}
\emph{Top:} median base-yaw trajectory, band $=$ IQR, $n{=}20$;
$\bullet$ median grasp, $\blacklozenge$ median release, dashed after it.
Release markers order with success --- ceiling $15.1^\circ$ (19/20) down to
$\thetaz$ $11.3^\circ$ (2/20), which also releases earliest; the grasp-phase
dispersion is the inherited 130-lineage spread, unrepaired by every operator.
\emph{Bottom:} each condition's \emph{medoid} end-state --- the real
terminal frame nearest that condition's median endpoint, one fixed rule,
nothing hand-picked (pixelwise averages: appendix); endpoint spread: tight
where a modal outcome exists (redirect 5\,px, ceiling/FT 8\,px), $3\times$
looser where none does ($\thetaz$, ascent 22\,px). The failure modes
differ: $\thetaz$ topples the assembly; ascent sets the cup \emph{beside}
the target. Blind-scored; Wilson CIs and taxonomy: Table~\ref{tab:trials}.}
\label{fig:robot}
\end{widefigure}

\subsection{The disagreement generalizes across policy classes}
\label{sec:generalize}
Prior cells replicate, and the $\pi_{0.5}$ arm is now measured on the same
per-demonstration instrument as ACT (Appendix~\ref{app:conv}; 3-seed
LoRA re-fit null $0.558 \pm 0.012$, calib/audit exchangeability $0.550$):
the robot checkpoints for redirect, FT, and ascent \emph{all} read rank AUC
$1.000$, CI $[1.00,1.00]$ --- identical to $\thetaz$ --- while the
within-model absolute level closes only a third of the $\thetaz\!\to$null
gap (\texttt{mem/null} $0.10 \to 0.34/0.34/0.44$). The redirect restoring
17/20 on hardware is again indistinguishable from the un-edited policy to the
rank audit (an earlier, weaker flow-space attack on the offline pilot
checkpoints told the same masking story); the arm's remaining scope limit
is behavioral, not evidential --- one task family at $n{=}20$. The DP cell
is measured
and replicates the overshoot exhibit in a third class: null $0.534$,
$\thetaz$ $0.780$, and delivered ascent at 500 steps reads rank AUC $0.560$
--- \emph{at} the null --- while its absolute member level runs
\texttt{mem/null} $=1.52$ past it.

The gate is
recipe-level (Table~\ref{tab:pushtgate};
Fig.~\ref{fig:pusht}): gap 14.8 points, seed-level Welch $t{=}3.78$,
$p\approx0.018$, every $\thetaz$ seed below every ceiling seed, and
\texttt{avg\_max\_reward} separates too (0.805 vs.\ 0.947) --- the mirrored
mode pulls toward the reflected target, not generic incompetence (seeds
four and five nearly doubled the 3-seed SD --- the fleet is the
instrument).

A control that must move nothing exposes a learning-rate cliff: a
$10\times$ lr change moved budget-matched FT's retain damage $175\times$,
voiding an entire first ascent sweep as an lr artifact
(Appendix~\ref{app:lrcliff}). All results below use the stable rate; the
pre-declared freeze rule selects ratio $0.20$, delivery stays
$0.79$--$0.83$, and the forget--retain gradient cosine is negative at every
dose --- ascent \emph{opposes} retain where redirect was orthogonal to it,
replicating DP.

The evidence axis rests on synthesized non-members (mirrors of clean
episodes disjoint from the forget sources and the operator's calibration
set) --- the arm's main internal risk, which survives its own check: five
retrain seeds give a null of $0.534 \pm 0.012$, exchangeable to a model
that saw neither.
Against it the three $\thetaz$ seeds read $0.981$--$0.997$, CIs disjoint from
every ceiling CI, at \texttt{mem/null} $0.088$--$0.109$: $\thetaz$ fits the
contaminated episodes ${\sim}10\times$ better than a model that never saw
them. Absolute levels are interval-honest (episode-bootstrap CI
half-width $\le 0.01$, side-stable against every retrain seed); a
calibrated two-sided attack corroborates (Appendix~\ref{app:sim}).

The operator results then reproduce, \emph{prospectively}, the rank-blindness
the ACT arm found retrospectively --- on all three $\thetaz$ seeds. At the
frozen dose FT is inert on both axes (rank $-0.001$ to $-0.003$ per seed;
behavior $-2.1$ points pooled, $t{=}{-}0.95$) --- P4 confirmed, matching
ACT's $+1.6\%$, $\pi_{0.5}$'s $4.6\%$ and DP's $0.19\%$ --- while ascent moves
the forget loss $+40$--$78\%$ per seed and buys $\le 0.003$ of rank AUC,
$4$--$6\%$ of the absolute evidence gap, and $+1.6$ points of behavior pooled
($t=1.39$, inside the null; Fig.~\ref{fig:pusht}).
P3 is refuted at the frozen dose: a large loss-space movement is
neither evidence erasure nor behavioral change.

At higher dose, ``repaired but not forgotten'' replicates across
seeds: auditing the over-budget rung the freeze rule
excluded (it breaks retain) recovers the central shape on every
$\thetaz$ seed --- at ratio
$0.50$ success moves $+8.0/{+}11.6/{+}2.4$ points (all three seeds toward the
fleet; Stouffer $p\approx0.004$), bought with retain damage of
$+42.7/{+}33.8/{+}68.7\%$ --- the guard stopped all three runs at
150/200 steps: not budget-matched, not deployable.
On the dose-resolved seed (s1000), that $39.5\%$-of-gap behavioral
recovery moves the rank audit $3.8\%$ and the absolute level $18.1\%$,
and the audit flags the rung
on \emph{all three} seeds (AUC $0.964/0.981/0.985$, every CI above the
null's highest seed $0.547$; \texttt{mem/null} $0.21$--$0.27$).
Retain was \emph{damaged} and success rose anyway --- repair by
asymmetric damage, not recovery --- and exploratory throughout
(Sec.~\ref{sec:design}); texture, not load-bearing.
The defensible claim: \emph{behavioral repair is purchasable with
retain damage; no tested dose brings either evidence statistic near the
retrain reference.}

\begin{widefigure}[t]
\centering
\includegraphics[width=0.44\linewidth]{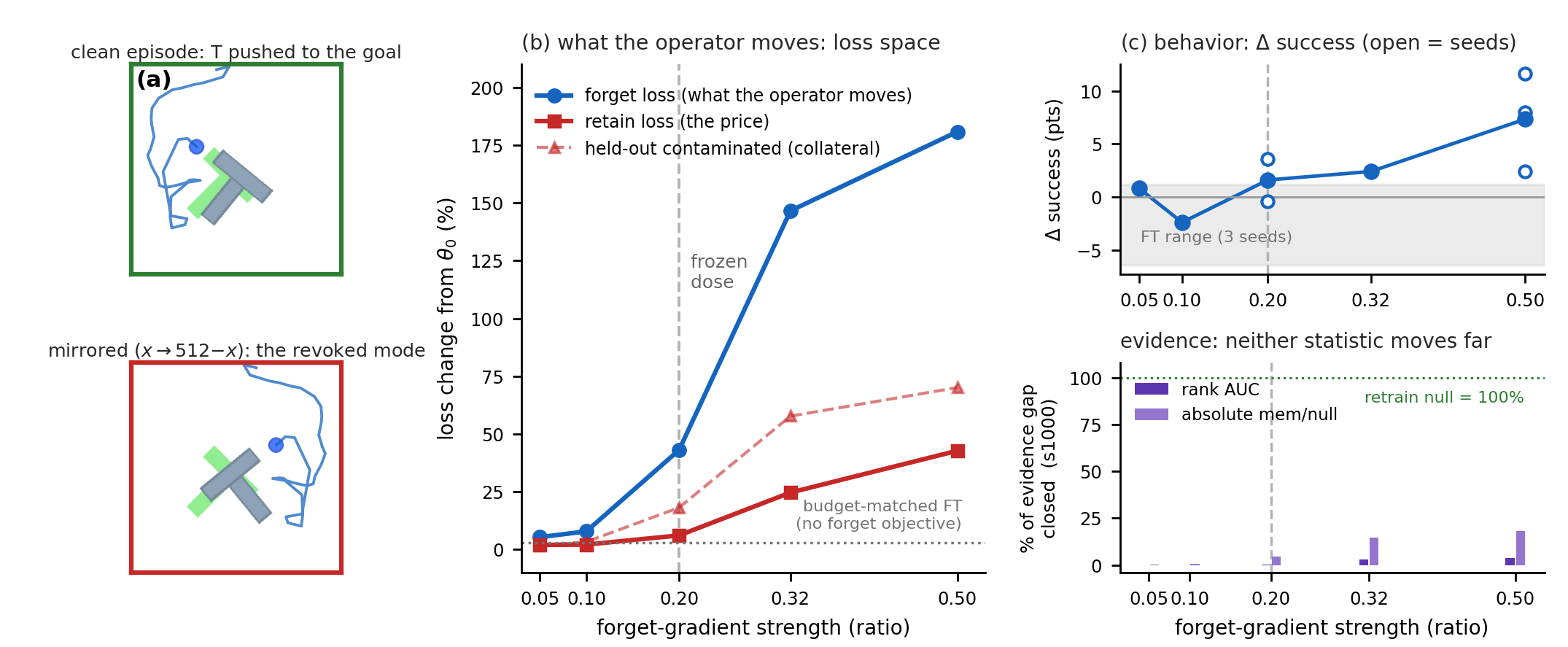}
\caption{\textbf{The PushT arm: the deletion request, what the operator
moves, and what the audits see.} \emph{(a)} gym-pusht at real dataset states: a clean episode ends with
the T on the green goal; its mirror ($x\to512{-}x$) pursues the reflected
pose --- a coherent wrong-goal mode. Blue: agent path. \emph{(b)} Ascent
moves what it optimizes (forget $+180.9\%$) at a climbing price (retain
$+42.7\%$, collateral $+70.1\%$; s1000); dashed: the frozen dose. \emph{(c)} Top: success change
per seed (r050: over-budget, exploratory); frozen-dose movement sits
inside FT's own seed range. Bottom: neither evidence
statistic approaches the null at any dose (rank $\le 3.8\%$, absolute
$\le 18.1\%$).}
\label{fig:pusht}
\end{widefigure}

\subsection{Applicability depends on the deletion request's structure}
\label{sec:applic}
Localized redirect fails a pre-specified feasibility gate --- a
boundary of the \emph{request}, not of the operator (bar fixed before the
intersection was computed, by dated amendment). Branch-redirect needs recorded \emph{actions} off the clean manifold at
\emph{states} still on it; whole-episode contamination has no such region
--- so under a frozen diagnostic (Appendix~\ref{app:feas},
Fig.~\ref{fig:feasibility}) branch windows cover 15.7\% of editable
frames but the support-gated intersection is 3.65\% --- under the 5\%
bar; \emph{tightening} the detector worsens it.
P2 is therefore prospectively not-applicable (an all-frames variant
would be global mode rewriting, outside scope) --- branch-redirect
requires contamination sharing a prefix with clean support, and not
every request supplies one.

\emph{Baselines.} A 14-method robomimic-BC comparison
(Appendix~\ref{app:robomimic}): 13 of 14 degrade closed-loop success
$68$--$101\%$ regardless of offline forget quality --- EU-$k$ pairs
near-perfect evidence ($0.008$) with a dead policy.
\section{Conclusion}
\label{sec:boundaries}
This paper turns ``unlearn my demonstrations'' into two measurable
guarantees: behavior against a retrain \emph{floor},
evidence against a retrain \emph{null} in rank and absolute form. They
come apart in both directions: hardware editing restores near-retrain
success at unchanged evidence; ascent's
endpoint is a training-draw lottery (null to $87\%$ past at identical
budgets). A 19-replica conformal audit rejects every edit at
$p{=}0.05$: deletion is a conjunction no one-axis audit certifies ---
yet the pair is constructive: applicability measurable, repair
purchasable at priced retain cost, an auditor holding both fooled by
neither masking nor overshoot.
Limits: repair is on-distribution only (deep-state attractors
survive; \citealp{ross2011dagger}); no offline metric predicted
closed-loop manifestation; composition fails; scope: mode-level
contamination of a re-runnable set (Appendix~\ref{app:gaps}).

\subsubsection*{Ethics statement}
All demonstrations were collected by the authors on their own hardware; no
third-party subjects, personal data, or personally identifying information
appear in any dataset, frame, or released artifact. The revocation scenario
is constructed: this work evaluates \emph{technical deletion proxies} ---
retrain-calibrated behavioral and membership audits --- and passing them
establishes neither legal compliance with a data-protection regime nor
causal removal of a demonstration's influence. We distinguish throughout
between regulatory deletion, causal removal, empirical retrain consistency,
behavioral repair, and attack-specific deniability; the paper's claims
concern only the last three, under the tested loss-based auditor family.

\subsubsection*{Reproducibility statement}
The released artifacts contain the dataset episode identities and every
train/retain/forget/calibration/audit split; the scripts that regenerate
every table and figure from measured JSON artifacts (no number in any table
is typed by hand); per-demonstration audit scores; the pre-registration and
amendment texts; the blind-scoring package (200 anonymized clips' scoring
sheet, rubric, and the sealed clip-to-condition mapping, opened only after
scoring); per-trial hardware outcome logs; and the configurations, seeds,
budgets, and stopping rules of every operator. Invalidated runs (the
unstable-lr sweep) are released alongside valid ones with their status
recorded. A manifest checker ships with the archive and verifies, item by
item, that a built archive contains everything this statement lists.

\bibliographystyle{iclr2027_conference}
\bibliography{refs}

\begin{thebibliography}{41}
\providecommand{\natexlab}[1]{#1}
\providecommand{\url}[1]{\texttt{#1}}
\expandafter\ifx\csname urlstyle\endcsname\relax
  \providecommand{\doi}[1]{doi: #1}\else
  \providecommand{\doi}{doi: \begingroup \urlstyle{rm}\Url}\fi

\bibitem[Agia et~al.(2025)Agia, Sinha, Yang, Antonova, Pavone, Nishimura,
  Itkina, and Bohg]{cupid2025}
Christopher Agia, Rohan Sinha, Jingyun Yang, Rika Antonova, Marco Pavone,
  Haruki Nishimura, Masha Itkina, and Jeannette Bohg.
\newblock {CUPID}: Curating data your robot loves with influence functions.
\newblock In \emph{Conference on Robot Learning (CoRL)}, volume 305 of
  \emph{PMLR}, pp.\  2907--2932, 2025.
\newblock URL \url{https://proceedings.mlr.press/v305/agia25a.html}.

\bibitem[Black et~al.(2025)Black, Brown, Driess, Esmail, Equi, Finn, Fusai,
  Groom, Hausman, Ichter, Jakubczak, Jones, Ke, Levine, Li-Bell, Mothukuri,
  Nair, Pertsch, Shi, Smith, Tanner, Vuong, Walling, Wang, and
  Zhilinsky]{black2024pi0}
Kevin Black, Noah Brown, Danny Driess, Adnan Esmail, Michael~Robert Equi,
  Chelsea Finn, Niccolo Fusai, Lachy Groom, Karol Hausman, Brian Ichter, Szymon
  Jakubczak, Tim Jones, Liyiming Ke, Sergey Levine, Adrian Li-Bell, Mohith
  Mothukuri, Suraj Nair, Karl Pertsch, Lucy~Xiaoyang Shi, Laura Smith, James
  Tanner, Quan Vuong, Anna Walling, Haohuan Wang, and Ury Zhilinsky.
\newblock {$\pi_0$: A Vision-Language-Action Flow Model for General Robot
  Control}.
\newblock In \emph{Robotics: Science and Systems (RSS)}, 2025.
\newblock \doi{10.15607/RSS.2025.XXI.010}.
\newblock arXiv:2410.24164.

\bibitem[Bourtoule et~al.(2021)Bourtoule, Chandrasekaran, Choquette-Choo, Jia,
  Travers, Zhang, Lie, and Papernot]{bourtoule2021sisa}
Lucas Bourtoule, Varun Chandrasekaran, Christopher~A. Choquette-Choo, Hengrui
  Jia, Adelin Travers, Baiwu Zhang, David Lie, and Nicolas Papernot.
\newblock Machine unlearning.
\newblock In \emph{IEEE Symposium on Security and Privacy (S\&P)}, 2021.

\bibitem[Cadene et~al.(2024)Cadene, Alibert, Soare, Gallouedec, Zouitine,
  Palma, Kooijmans, Aractingi, Shukor, Aubakirova, Russi, Capuano, Pascal,
  Choghari, Meftah, Ellerbach, Moss, and Wolf]{cadene2024lerobot}
Remi Cadene, Simon Alibert, Alexander Soare, Quentin Gallouedec, Adil Zouitine,
  Steven Palma, Pepijn Kooijmans, Michel Aractingi, Mustafa Shukor, Dana
  Aubakirova, Martino Russi, Francesco Capuano, Caroline Pascal, Jade Choghari,
  Khalil Meftah, Maxime Ellerbach, Jess Moss, and Thomas Wolf.
\newblock {LeRobot}: State-of-the-art machine learning for real-world robotics
  in {Pytorch}.
\newblock \url{https://github.com/huggingface/lerobot}, 2024.

\bibitem[Carlini et~al.(2022)Carlini, Chien, Nasr, Song, Terzis, and
  Tram{\`e}r]{carlini2022lira}
Nicholas Carlini, Steve Chien, Milad Nasr, Shuang Song, Andreas Terzis, and
  Florian Tram{\`e}r.
\newblock Membership inference attacks from first principles.
\newblock In \emph{IEEE Symposium on Security and Privacy (S\&P)}, 2022.

\bibitem[Chi et~al.(2023)Chi, Feng, Du, Xu, Cousineau, Burchfiel, and
  Song]{chi2023diffusion}
Cheng Chi, Siyuan Feng, Yilun Du, Zhenjia Xu, Eric Cousineau, Benjamin
  Burchfiel, and Shuran Song.
\newblock Diffusion policy: Visuomotor policy learning via action diffusion.
\newblock In \emph{Robotics: Science and Systems (RSS)}, 2023.

\bibitem[Chi et~al.(2024)Chi, Xu, Pan, Cousineau, Burchfiel, Feng, Tedrake, and
  Song]{chi2024umi}
Cheng Chi, Zhenjia Xu, Chuer Pan, Eric Cousineau, Benjamin Burchfiel, Siyuan
  Feng, Russ Tedrake, and Shuran Song.
\newblock Universal manipulation interface: In-the-wild robot teaching without
  in-the-wild robots.
\newblock In \emph{Robotics: Science and Systems (RSS)}, 2024.

\bibitem[Dass et~al.(2026)Dass, Khaddaj, Engstrom, Madry, Ilyas, and
  Mart{\'i}n-Mart{\'i}n]{datamil2025}
Shivin Dass, Alaa Khaddaj, Logan Engstrom, Aleksander Madry, Andrew Ilyas, and
  Roberto Mart{\'i}n-Mart{\'i}n.
\newblock {DataMIL}: Selecting data for robot imitation learning with
  datamodels.
\newblock In \emph{International Conference on Learning Representations
  (ICLR)}, 2026.
\newblock URL \url{https://arxiv.org/abs/2505.09603}.

\bibitem[Eldan \& Russinovich(2023)Eldan and Russinovich]{eldan2023whos}
Ronen Eldan and Mark Russinovich.
\newblock Who's {Harry Potter}? approximate unlearning in {LLMs}.
\newblock \emph{arXiv preprint arXiv:2310.02238}, 2023.

\bibitem[Fan et~al.(2024)Fan, Liu, Zhang, Wong, Wei, and Liu]{fan2024salun}
Chongyu Fan, Jiancheng Liu, Yihua Zhang, Eric Wong, Dennis Wei, and Sijia Liu.
\newblock {SalUn}: Empowering machine unlearning via gradient-based weight
  saliency in both image classification and generation.
\newblock In \emph{ICLR}, 2024.

\bibitem[Ginart et~al.(2019)Ginart, Guan, Valiant, and Zou]{ginart2019making}
Antonio Ginart, Melody Guan, Gregory Valiant, and James Zou.
\newblock Making {AI} forget you: Data deletion in machine learning.
\newblock In \emph{NeurIPS}, 2019.

\bibitem[Goel et~al.(2022)Goel, Prabhu, Sanyal, Lim, Torr, and
  Kumaraguru]{goel2022evaluating}
Shashwat Goel, Ameya Prabhu, Amartya Sanyal, Ser-Nam Lim, Philip Torr, and
  Ponnurangam Kumaraguru.
\newblock Towards adversarial evaluations for inexact machine unlearning.
\newblock \emph{arXiv preprint arXiv:2201.06640}, 2022.
\newblock URL \url{https://arxiv.org/abs/2201.06640}.

\bibitem[Golatkar et~al.(2020)Golatkar, Achille, and
  Soatto]{golatkar2020eternal}
Aditya Golatkar, Alessandro Achille, and Stefano Soatto.
\newblock Eternal sunshine of the spotless net: Selective forgetting in deep
  networks.
\newblock In \emph{CVPR}, 2020.

\bibitem[Gong et~al.(2025)Gong, Li, Yao, and Wang]{trajdeleter2024}
Chen Gong, Kecen Li, Jin Yao, and Tianhao Wang.
\newblock {TrajDeleter}: Enabling trajectory forgetting in offline
  reinforcement learning agents.
\newblock In \emph{Network and Distributed System Security Symposium (NDSS)},
  2025.
\newblock URL \url{https://arxiv.org/abs/2404.12530}.

\bibitem[Guo et~al.(2020)Guo, Goldstein, Hannun, and van~der
  Maaten]{guo2020certified}
Chuan Guo, Tom Goldstein, Awni Hannun, and Laurens van~der Maaten.
\newblock Certified data removal from machine learning models.
\newblock In \emph{International Conference on Machine Learning (ICML)}, 2020.

\bibitem[Hu et~al.(2021)Hu, Shen, Wallis, Allen-Zhu, Li, Wang, Wang, and
  Chen]{hu2021lora}
Edward~J. Hu, Yelong Shen, Phillip Wallis, Zeyuan Allen-Zhu, Yuanzhi Li, Shean
  Wang, Lu~Wang, and Weizhu Chen.
\newblock {LoRA}: Low-rank adaptation of large language models, 2021.
\newblock arXiv:2106.09685.

\bibitem[Kelly et~al.(2026)Kelly, Zhang, Zhang, Lin, and Li]{rff2025}
Manuel Kelly, Xin Zhang, Yingxue Zhang, Fangzhou Lin, and Yanhua Li.
\newblock Unlearning diffusion policies with relative fisher forgetting.
\newblock OpenReview preprint, 2026.
\newblock URL \url{https://openreview.net/forum?id=TidLO0qdp0}.
\newblock Non-archival; submitted to ICLR 2026.

\bibitem[Kim et~al.(2024)Kim, Pertsch, Karamcheti, Xiao, Balakrishna, Nair,
  Rafailov, Foster, Lam, Sanketi, Vuong, Kollar, Burchfiel, Tedrake, Sadigh,
  Levine, Liang, and Finn]{kim2024openvla}
Moo~Jin Kim, Karl Pertsch, Siddharth Karamcheti, Ted Xiao, Ashwin Balakrishna,
  Suraj Nair, Rafael Rafailov, Ethan Foster, Grace Lam, Pannag Sanketi, Quan
  Vuong, Thomas Kollar, Benjamin Burchfiel, Russ Tedrake, Dorsa Sadigh, Sergey
  Levine, Percy Liang, and Chelsea Finn.
\newblock {OpenVLA}: An open-source vision-language-action model.
\newblock In \emph{Conference on Robot Learning (CoRL)}, 2024.
\newblock URL \url{https://arxiv.org/abs/2406.09246}.

\bibitem[Koh \& Liang(2017)Koh and Liang]{koh2017understanding}
Pang~Wei Koh and Percy Liang.
\newblock Understanding black-box predictions via influence functions.
\newblock In \emph{ICML}, 2017.

\bibitem[Kurmanji et~al.(2023)Kurmanji, Triantafillou, Hayes, and
  Triantafillou]{kurmanji2023scrub}
Meghdad Kurmanji, Peter Triantafillou, Jamie Hayes, and Eleni Triantafillou.
\newblock Towards unbounded machine unlearning.
\newblock In \emph{NeurIPS}, 2023.

\bibitem[Li et~al.(2024)Li, Pan, Gopal, Yue, Berrios, Gatti, Li, Dombrowski,
  Goel, Mukobi, Helm-Burger, Lababidi, Justen, Liu, Chen, Barrass, Zhang, Zhu,
  Tamirisa, Bharathi, Herbert-Voss, Breuer, Zou, Mazeika, Wang, Oswal, Lin,
  Hunt, Tienken-Harder, Shih, Talley, Guan, Steneker, Campbell, Jokubaitis,
  Basart, Fitz, Kumaraguru, Karmakar, Tupakula, Varadharajan, Shoshitaishvili,
  Ba, Esvelt, Wang, and Hendrycks]{li2024wmdp}
Nathaniel Li, Alexander Pan, Anjali Gopal, Summer Yue, Daniel Berrios, Alice
  Gatti, Justin~D. Li, Ann-Kathrin Dombrowski, Shashwat Goel, Gabriel Mukobi,
  Nathan Helm-Burger, Rassin Lababidi, Lennart Justen, Andrew~Bo Liu, Michael
  Chen, Isabelle Barrass, Oliver Zhang, Xiaoyuan Zhu, Rishub Tamirisa, Bhrugu
  Bharathi, Ariel Herbert-Voss, Cort~B. Breuer, Andy Zou, Mantas Mazeika, Zifan
  Wang, Palash Oswal, Weiran Lin, Adam~Alfred Hunt, Justin Tienken-Harder,
  Kevin~Y. Shih, Kemper Talley, John Guan, Ian Steneker, David Campbell, Brad
  Jokubaitis, Steven Basart, Stephen Fitz, Ponnurangam Kumaraguru,
  Kallol~Krishna Karmakar, Uday Tupakula, Vijay Varadharajan, Yan
  Shoshitaishvili, Jimmy Ba, Kevin~M. Esvelt, Alexandr Wang, and Dan Hendrycks.
\newblock The {WMDP} benchmark: Measuring and reducing malicious use with
  unlearning.
\newblock In \emph{International Conference on Machine Learning (ICML)}, volume
  235 of \emph{PMLR}, pp.\  28525--28550, 2024.
\newblock URL \url{https://proceedings.mlr.press/v235/li24bc.html}.

\bibitem[Maini et~al.(2024)Maini, Feng, Schwarzschild, Lipton, and
  Kolter]{maini2024tofu}
Pratyush Maini, Zhili Feng, Avi Schwarzschild, Zachary~C. Lipton, and J.~Zico
  Kolter.
\newblock {TOFU}: A task of fictitious unlearning for {LLMs}, 2024.
\newblock arXiv:2401.06121.

\bibitem[Mandlekar et~al.(2021)Mandlekar, Xu, Wong, Nasiriany, Wang, Kulkarni,
  Fei-Fei, Savarese, Zhu, and Mart{\'i}n-Mart{\'i}n]{mandlekar2021robomimic}
Ajay Mandlekar, Danfei Xu, Josiah Wong, Soroush Nasiriany, Chen Wang, Rohun
  Kulkarni, Li~Fei-Fei, Silvio Savarese, Yuke Zhu, and Roberto
  Mart{\'i}n-Mart{\'i}n.
\newblock What matters in learning from offline human demonstrations for robot
  manipulation.
\newblock In \emph{Conference on Robot Learning (CoRL)}, 2021.

\bibitem[{Physical Intelligence} et~al.(2025){Physical Intelligence}, Black,
  Brown, Darpinian, Dhabalia, Driess, Esmail, Equi, Finn, Fusai, Galliker,
  Ghosh, Groom, Hausman, Ichter, Jakubczak, Jones, Ke, LeBlanc, Levine,
  Li-Bell, Mothukuri, Nair, Pertsch, Ren, Shi, Smith, Springenberg, Stachowicz,
  Tanner, Vuong, Walke, Walling, Wang, Yu, and Zhilinsky]{intelligence2025pi05}
{Physical Intelligence}, Kevin Black, Noah Brown, James Darpinian, Karan
  Dhabalia, Danny Driess, Adnan Esmail, Michael Equi, Chelsea Finn, Niccolo
  Fusai, Manuel~Y. Galliker, Dibya Ghosh, Lachy Groom, Karol Hausman, Brian
  Ichter, Szymon Jakubczak, Tim Jones, Liyiming Ke, Devin LeBlanc, Sergey
  Levine, Adrian Li-Bell, Mohith Mothukuri, Suraj Nair, Karl Pertsch, Allen~Z.
  Ren, Lucy~Xiaoyang Shi, Laura Smith, Jost~Tobias Springenberg, Kyle
  Stachowicz, James Tanner, Quan Vuong, Homer Walke, Anna Walling, Haohuan
  Wang, Lili Yu, and Ury Zhilinsky.
\newblock {$\pi_{0.5}$}: a vision-language-action model with open-world
  generalization.
\newblock In \emph{Conference on Robot Learning (CoRL)}, volume 305 of
  \emph{PMLR}, pp.\  17--40, 2025.
\newblock URL \url{https://arxiv.org/abs/2504.16054}.

\bibitem[Pomerleau(1989)]{pomerleau1989alvinn}
Dean~A. Pomerleau.
\newblock {ALVINN}: An autonomous land vehicle in a neural network.
\newblock In \emph{NeurIPS}, 1989.

\bibitem[Ranjan \& Polyzou(2026)Ranjan and Polyzou]{vlaforget2026}
Ravi Ranjan and Agoritsa Polyzou.
\newblock {VLA-Forget}: Vision-language-action unlearning for embodied
  foundation models.
\newblock In \emph{Proceedings of the 4th Workshop on Towards Knowledgeable
  Foundation Models (KnowFM)}, pp.\  60--77. Association for Computational
  Linguistics, 2026.
\newblock URL \url{https://aclanthology.org/2026.knowfm-1.5/}.
\newblock arXiv:2604.03956v2.

\bibitem[Ross et~al.(2011)Ross, Gordon, and Bagnell]{ross2011dagger}
St{\'e}phane Ross, Geoffrey Gordon, and Drew Bagnell.
\newblock A reduction of imitation learning and structured prediction to
  no-regret online learning.
\newblock In \emph{AISTATS}, 2011.

\bibitem[Salem et~al.(2019)Salem, Zhang, Humbert, Berrang, Fritz, and
  Backes]{salem2019mlleaks}
Ahmed Salem, Yang Zhang, Mathias Humbert, Pascal Berrang, Mario Fritz, and
  Michael Backes.
\newblock {ML-Leaks}: Model and data independent membership inference attacks
  and defenses on machine learning models.
\newblock In \emph{NDSS}, 2019.

\bibitem[Sekhari et~al.(2021)Sekhari, Acharya, Kamath, and
  Suresh]{sekhari2021remember}
Ayush Sekhari, Jayadev Acharya, Gautam Kamath, and Ananda~Theertha Suresh.
\newblock Remember what you want to forget: Algorithms for machine unlearning.
\newblock In \emph{NeurIPS}, 2021.

\bibitem[Shokri et~al.(2017)Shokri, Stronati, Song, and
  Shmatikov]{shokri2017membership}
Reza Shokri, Marco Stronati, Congzheng Song, and Vitaly Shmatikov.
\newblock Membership inference attacks against machine learning models.
\newblock In \emph{IEEE Symposium on Security and Privacy (S\&P)}, 2017.

\bibitem[Snyder et~al.(2025)Snyder, Hancock, Badithela, Dixon, Miller, Ambrus,
  Majumdar, Itkina, and Nishimura]{step2025stopping}
David Snyder, Asher~James Hancock, Apurva Badithela, Emma Dixon, Patrick
  Miller, Rares~Andrei Ambrus, Anirudha Majumdar, Masha Itkina, and Haruki
  Nishimura.
\newblock Is your imitation learning policy better than mine? policy comparison
  with near-optimal stopping.
\newblock In \emph{Robotics: Science and Systems (RSS)}, 2025.
\newblock URL \url{https://arxiv.org/abs/2503.10966}.

\bibitem[Tarun et~al.(2023)Tarun, Chundawat, Mandal, and
  Kankanhalli]{tarun2023deep}
Ayush~Kumar Tarun, Vikram~Singh Chundawat, Murari Mandal, and Mohan
  Kankanhalli.
\newblock Deep regression unlearning.
\newblock In \emph{ICML}, 2023.

\bibitem[{TRI LBM Team}(2025)]{tri2025lbm}
{TRI LBM Team}.
\newblock A careful examination of large behavior models for multitask
  dexterous manipulation.
\newblock \emph{arXiv preprint arXiv:2507.05331}, 2025.
\newblock URL \url{https://arxiv.org/abs/2507.05331}.

\bibitem[Wang et~al.(2026)Wang, Lei, Li, Liu, Zheng, Li, Zhang, and
  Fan]{more2026}
Hao Wang, Jiuzhou Lei, Dayou Li, Bangya Liu, Minghui Zheng, Manling Li, Ruohan
  Zhang, and Zhiwen Fan.
\newblock Behavior uncloning: Distilling mode redirection into policy weights
  without inference-time steering.
\newblock \emph{arXiv preprint arXiv:2606.29201}, 2026.
\newblock URL \url{https://arxiv.org/abs/2606.29201}.

\bibitem[Yan et~al.(2026)Yan, Li, Zhu, Wang, Shou, Miao, Pang, Hong, and
  Guo]{redflow2026}
Zhengyang Yan, Junhao Li, Fangqi Zhu, Zijun Wang, Quanxin Shou, Yikun Miao,
  Xiaoyi Pang, Zicong Hong, and Song Guo.
\newblock {RedFlow}: Redirect failure into action-level corrections for
  flow-matching {VLA} policy.
\newblock \emph{arXiv preprint arXiv:2607.27782}, 2026.
\newblock URL \url{https://arxiv.org/abs/2607.27782}.

\bibitem[Yeom et~al.(2018)Yeom, Giacomelli, Fredrikson, and
  Jha]{yeom2018privacy}
Samuel Yeom, Irene Giacomelli, Matt Fredrikson, and Somesh Jha.
\newblock Privacy risk in machine learning: Analyzing the connection to
  overfitting.
\newblock In \emph{IEEE Computer Security Foundations Symposium (CSF)}, 2018.

\bibitem[Zhang(2024)]{zhang2024partial}
Jiahao Zhang.
\newblock Graph unlearning with efficient partial retraining.
\newblock In \emph{Companion Proceedings of the ACM Web Conference (WWW), PhD
  Symposium Track}, pp.\  1218--1221, 2024.
\newblock \doi{10.1145/3589335.3651265}.
\newblock arXiv:2403.07353.

\bibitem[Zhang et~al.(2026{\natexlab{a}})Zhang, Wang, and
  Wang]{zhang2026attackby}
Jiahao Zhang, Yilong Wang, and Suhang Wang.
\newblock Attack by unlearning: Unlearning-induced adversarial attacks on graph
  neural networks, 2026{\natexlab{a}}.
\newblock URL \url{https://arxiv.org/abs/2603.18570}.
\newblock arXiv:2603.18570.

\bibitem[Zhang et~al.(2026{\natexlab{b}})Zhang, Wang, Zhang, Liu, and
  Wang]{zhang2026inversion}
Jiahao Zhang, Yilong Wang, Zhiwei Zhang, Xiaorui Liu, and Suhang Wang.
\newblock Unlearning inversion attacks for graph neural networks.
\newblock In \emph{ACM International Conference on Web Search and Data Mining
  (WSDM)}, pp.\  934--945, 2026{\natexlab{b}}.
\newblock \doi{10.1145/3773966.3777929}.
\newblock arXiv:2506.00808.

\bibitem[Zhang et~al.(2024)Zhang, Lin, Bai, and Mei]{zhang2024npo}
Ruiqi Zhang, Licong Lin, Yu~Bai, and Song Mei.
\newblock Negative preference optimization: From catastrophic collapse to
  effective unlearning, 2024.
\newblock arXiv:2404.05868.

\bibitem[Zhao et~al.(2023)Zhao, Kumar, Levine, and Finn]{zhao2023act}
Tony~Z. Zhao, Vikash Kumar, Sergey Levine, and Chelsea Finn.
\newblock Learning fine-grained bimanual manipulation with low-cost hardware.
\newblock In \emph{Robotics: Science and Systems (RSS)}, 2023.

\end{thebibliography}

\appendix
\section{Measurement conventions}
\label{app:conv}
Every action-space number in the paper carries the signature
$[\text{units},\ \text{reference},\ \text{region},\ \text{horizon},\
\text{samples}]$. The per-arm instantiations:

\begin{table}[ht]
\centering
\scriptsize
\setlength{\tabcolsep}{4pt}
\caption{\textbf{Per-arm measurement conventions.}
$^\dagger$ACT's primary horizon is the executed 25 steps of each 50-step
chunk; the full-plan figures are the late-ramp diagnostic
(Appendix~\ref{app:floors}, Sec.~\ref{sec:offline}). Cross-arm comparisons
use a matched 8-step slice. Units name the axis: action-space units read
counterfactual conduct (ACT, DP, $\pi_{0.5}$); PushT's closed-loop
success reads the manifestation channel only. Membership audits: per-demonstration loss attack,
mean aggregation, 2000-resample bootstrap CIs; members/non-members 30/10 (ACT,
DP, $\pi_{0.5}$; 50-set episodes 0--29 vs.\ 40--49) and 80/80 for PushT
(synthesized mirror non-members, Sec.~\ref{sec:generalize}); nulls pool all
retrain seeds of the arm --- for ACT, 5 seeds $\times$ 3 frame draws
(Sec.~\ref{sec:offline}).}
\label{tab:conventions}
\begin{tabular}{@{}llllll@{}}
\toprule
arm & units & reference & regions & horizon & samples \\
\midrule
ACT & raw joint $L_1$ & 3 retrains (3 pairs) & \scriptsize PRE/BRANCH/POST/TAIL &
25 executed$^\dagger$ & 1200 banked obs \\
DP & raw action $L_1$ & 4 retrains (6 pairs) & \scriptsize PRE/BRANCH/POST/TAIL &
8 (executed) & banked obs \\
$\pi_{0.5}$ & action MAE / flow loss & 3 LoRA re-fits & --- & 50 (executed) &
96 samples/split \\
PushT & closed-loop success & 5-seed fleet & --- & episode & 250 rollouts/cond. \\
\bottomrule
\end{tabular}

\end{table}

\section{Floor tables, both horizon conventions}
\label{app:floors}
Raw-unit divergence to the ceiling per region, under the executed-slice
convention (primary; what the main text reports) and the full 50-step plan
(the late-ramp diagnostic), both recomputed from the same banks by \texttt{act\_slice\_report.py}; the recomputation reproduces
the published full-50 floor and per-model BRANCH cells to $<5\times10^{-3}$
before slicing. A pooled action-space magnitude does not tell you \emph{where}
a policy diverges --- the same summary covers a path difference and a placement
difference --- which is why every number is region-resolved and the temporal
decomposition (Fig.~\ref{fig:manifest}) is reported alongside.

\begin{widetable}[ht]
\centering
\scriptsize
\setlength{\tabcolsep}{2.2pt}
\caption{\textbf{ACT c50 divergence to ceiling}, raw joint units, forget
regions, both horizon conventions. The $\thetaz$ BRANCH gap over the floor
is $2.2$ raw units on the executed slice vs.\ $7.9$ full-plan; operator
ordering is identical under both.}
\label{tab:floors}
\begin{tabular}{@{}l cccc c cccc@{}}
\toprule
 & \multicolumn{4}{c}{executed slice (25 steps, primary)} & & \multicolumn{4}{c}{full plan (50 steps, diagnostic)} \\
\cmidrule(lr){2-5}\cmidrule(lr){7-10}
 & BRANCH & POST & PRE & TAIL & & BRANCH & POST & PRE & TAIL \\
\midrule
floor (3 pairs) & $1.81\pm0.07$ & $2.01\pm0.08$ & $2.83\pm0.14$ & $7.27\pm0.97$ & & $1.83\pm0.06$ & $3.27\pm0.03$ & $3.44\pm0.11$ & $8.93\pm0.64$ \\
\addlinespace[2pt]
$\thetaz$ & 3.99 & 13.25 & 4.48 & 16.51 & & 9.70 & 16.86 & 6.03 & 22.73 \\
R200 & 3.45 & 7.21 & 4.10 & 15.90 & & 3.79 & 9.02 & 5.09 & 22.12 \\
R400 & 3.48 & 6.27 & 3.84 & 16.37 & & 3.80 & 8.06 & 4.72 & 22.76 \\
B50 & 3.96 & 14.09 & 4.54 & 20.58 & & 10.49 & 20.14 & 6.02 & 25.82 \\
B100 & 4.30 & 14.28 & 4.80 & 21.72 & & 11.96 & 20.46 & 6.45 & 26.93 \\
B200 & 5.25 & 15.84 & 5.70 & 23.44 & & 13.83 & 22.21 & 7.60 & 28.60 \\
B300 & 8.01 & 19.26 & 9.90 & 29.12 & & 19.82 & 26.84 & 13.73 & 33.49 \\
B400 & 7.97 & 18.68 & 11.56 & 30.15 & & 19.98 & 26.28 & 16.18 & 34.43 \\
FT200 & 3.92 & 12.23 & 4.31 & 16.12 & & 9.58 & 16.04 & 5.72 & 22.24 \\
\bottomrule
\end{tabular}

\end{widetable}

\section{The PushT learning-rate cliff, the voided grid, and the seed
replication}
\label{app:lrcliff}
The budget-matched FT control exposed the instability: a $10\times$ lr change
moves its retain damage $175\times$, so the entire first ascent sweep at
$10^{-5}$ (guard-stopped at 100/200 steps on every rung) is an lr artifact
and survives only as the exhibit below. All main-text PushT numbers use the
stable rate $2.5\times10^{-6}$.

\begin{table}[ht]
\centering
\footnotesize
\caption{\textbf{The lr cliff (s1000).} Loss changes from $\thetaz$.}
\label{tab:lrcliff}
\begin{tabular*}{\linewidth}{@{\extracolsep{\fill}}llccc@{}}
\toprule
 & & retain & forget & collateral \\
\midrule
\multicolumn{5}{@{}l}{\makebox[0pt][l]{\textbf{FT control vs.\ learning rate} \footnotesize (no forget objective; must move nothing)}}\\[1pt]
\quad lr $10^{-6}$ & steps 200/200 & +0.07\% & +1.62\% & +1.59\% \\
\quad lr $2.5\times10^{-6}$ & steps 200/200 & +2.15\% & +2.77\% & +1.86\% \\
\quad lr $10^{-5}$ & steps 200/200 & +12.29\% & +9.51\% & +3.50\% \\
\addlinespace[2pt]
\multicolumn{5}{@{}l}{\makebox[0pt][l]{\textbf{the voided $10^{-5}$ ascent grid} \footnotesize (guard-stopped at 100/200 steps; lr artifact, not an operator property)}}\\[1pt]
\quad ratio 0.20 & steps 100/200 & +40.3\% & +325.3\% & +126.3\% \\
\quad ratio 0.32 & steps 100/200 & +72.9\% & +403.9\% & +158.1\% \\
\quad ratio 0.50 & steps 100/200 & +110.4\% & +356.0\% & +145.1\% \\
\bottomrule
\end{tabular*}

\end{table}

\begin{table}[ht]
\centering
\footnotesize
\caption{\textbf{Three-seed replication at the stable rate.} Success deltas
are against each seed's own $\thetaz$ (250 rollouts); losses from the edit
reports. Ratio 0.50 moves success toward the fleet on all three seeds and is
guard-stopped at 150/200 steps on all three --- not budget-matched, not
deployable. The dose-resolved MIA ladder exists for s1000; the r050
endpoint is measured on all three seeds ($0.964/0.981/0.985$).}
\label{tab:pushtseeds}
\begin{tabular}{@{}llcccc@{}}
\toprule
seed & operator & $\Delta$success (pts) & retain & forget & steps \\
\midrule
s1000 & ascent @0.20 (frozen) & -0.4 & +6.1\% & +43.1\% & 200/200 \\
s1000 & ascent @0.50 & +8.0 & +42.7\% & +180.9\% & 150/200 \\
s1000 & FT200 & +1.2 & +2.2\% & +2.8\% & 200/200 \\
\addlinespace[2pt]
s2000 & ascent @0.20 (frozen) & +3.6 & +11.3\% & +40.4\% & 200/200 \\
s2000 & ascent @0.50 & +11.6 & +33.8\% & +141.0\% & 150/200 \\
s2000 & FT200 & -1.2 & +3.6\% & +3.7\% & 200/200 \\
\addlinespace[2pt]
s3000 & ascent @0.20 (frozen) & +1.6 & +19.1\% & +77.5\% & 200/200 \\
s3000 & ascent @0.50 & +2.4 & +68.7\% & +249.6\% & 150/200 \\
s3000 & FT200 & -6.4 & +5.0\% & +1.5\% & 200/200 \\
\bottomrule
\end{tabular}

\end{table}

\section{Simulation arms}
\label{app:sim}
\emph{PushT (diffusion, keypoints).} Base data: the 206-episode
\texttt{pusht\_keypoints} set. Mirror contamination: 80 source episodes
(seed 20260813) reflected $x \to 512{-}x$ on agent, action, and all 8 T
keypoints with slot permutation $[1,0,3,2,5,4,7,6]$; the permutation is
validated data-only --- mirror+perm of real frames is rotation-congruent
(det $=+1$) to real labeled frames at Kabsch residual 0.0000. Mirrored
episodes carry \texttt{next.reward}${=}0$ / \texttt{next.success}${=}$False by
construction (never read by BC; recorded for schema fidelity). $\thetaz$
trains on all 286 (3 seeds); the ceiling on the 206 clean episodes (5 seeds);
both share one norm-stats provenance chain. Gate: Table~\ref{tab:pushtgate}.
Non-member synthesis for the membership audit mirrors clean episodes drawn
disjoint from both the forget sources and the operator's calibration set
(Sec.~\ref{sec:generalize}).

\emph{Second auditor (calibrated, two-sided).} A reference-model attack in
the LiRA family scores each member episode by its loss $z$-scored against
the five retrain seeds, so an episode hard for \emph{every} model stops
counting as membership signal --- per-episode calibration, two-sided by
construction (outcome (3) is visible to it). On the measured artifacts:
every $\thetaz$ seed reads mean $z \approx -7$ ($100\%$ of member episodes
beyond $|z|{>}1.96$, all on the member-like side; leave-one-out retrain
nulls span $-1.1$ to $+3.4$), FT is unchanged, and the exploratory high
dose moves the signal only to $-5.7$ --- corroborating, with a second
auditor family, both the frozen-dose inertness and the still-flagged high
dose. On ACT (five reference retrains) the same instrument reads
$\thetaz$ $-14.4$, FT $-14.9$, redirect $-10.8$/$-10.1$ --- $100\%$ of
member episodes flagged member-like, so redirect's masking survives
per-episode calibration --- while the primary lineage's ascent ladder
steps \emph{over} the null without landing: $-11.4 \to -3.1$ (B200) $\to +10.3$ (B300, $100\%$
flagged on the overshoot side) $\to +14.7$; leave-one-out retrain nulls
span $-1.7$ to $+2.0$. The instrument ships for the $\pi_{0.5}$
per-episode artifacts as well (\texttt{mia\_second\_attack.py}).

\emph{robomimic-BC (Can/Square, MH).} corr60: $\thetaz$ trains on 240
demonstrations of which 60 carry the constructed worse-operator contamination;
calib and audit holdouts are disjoint. The natural split uses observed
worse-operator labels instead. Fig.~\ref{fig:envs} shows both environments at
real demonstration states.

\begin{figure}[ht]
\centering
\includegraphics[width=0.36\linewidth]{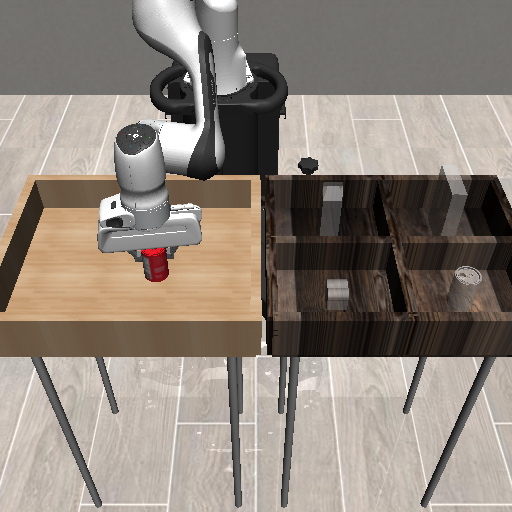}\hspace{0.02\linewidth}
\includegraphics[width=0.36\linewidth]{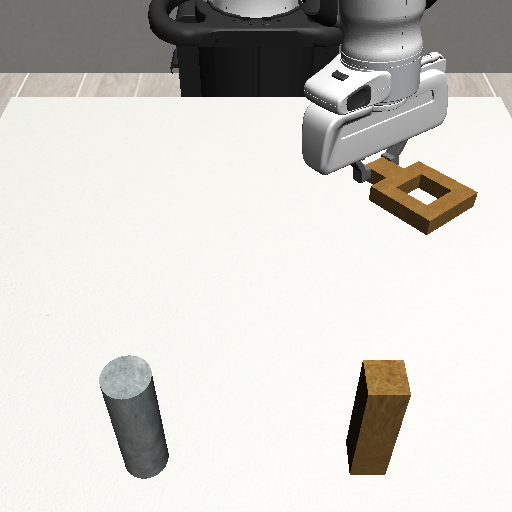}
\caption{\textbf{The robomimic-BC environments}, rendered by robomimic itself
at the mid-episode state of each dataset's first demonstration (fixed rule,
nothing hand-picked): PickPlaceCan (left) and NutAssemblySquare (right).}
\label{fig:envs}
\end{figure}

\begin{table}[ht]
\centering
\caption{\textbf{PushT gate, as fleets}: 250 rollouts per seed. The
contaminated \emph{recipe} sits 14.8 points below the retrain recipe
(seed-level Welch $t{=}3.78$, $p\approx0.018$); every $\thetaz$ seed falls
below every ceiling seed.}
\label{tab:pushtgate}
\footnotesize
\setlength{\tabcolsep}{4pt}
\begin{tabular}{@{}llc@{}}
\toprule
 & success per seed (\%) & mean $\pm$ SD \\
\midrule
ceiling (5 retrains) & 75.6 / 76.0 / 70.4 / 64.4 / 76.8 & $72.6 \pm 5.25$ \\
$\thetaz$ (3 seeds)  & 52.4 / 58.0 / 63.2               & $57.9 \pm 5.40$ \\
\bottomrule
\end{tabular}
\end{table}

\section{Redirect-feasibility diagnostic}
\label{app:feas}
The frozen diagnostic (config and outputs archived in
\path{redirect_feasibility_20260814.json}) scores every editable forget
frame on two channels: action discrepancy against the clean action manifold
(leave-one-out calibrated on clean episodes) and state-support distance to
the clean state distribution. Branch-redirect requires frames high in the
first and low in the second; Fig.~\ref{fig:feasibility} shows mirror
contamination supplies almost none, and the 5\% applicability bar (set from
ACT's measured 8.6\% BRANCH fraction, fixed by dated amendment before the
intersection was computed) fails at every detector calibration. Recorded
decision: fail localized redirect on mirror contamination; report the
all-frames variant only as a global mode-rewriting scope ablation.

\begin{widefigure}[ht]
\centering
\includegraphics[width=0.98\linewidth]{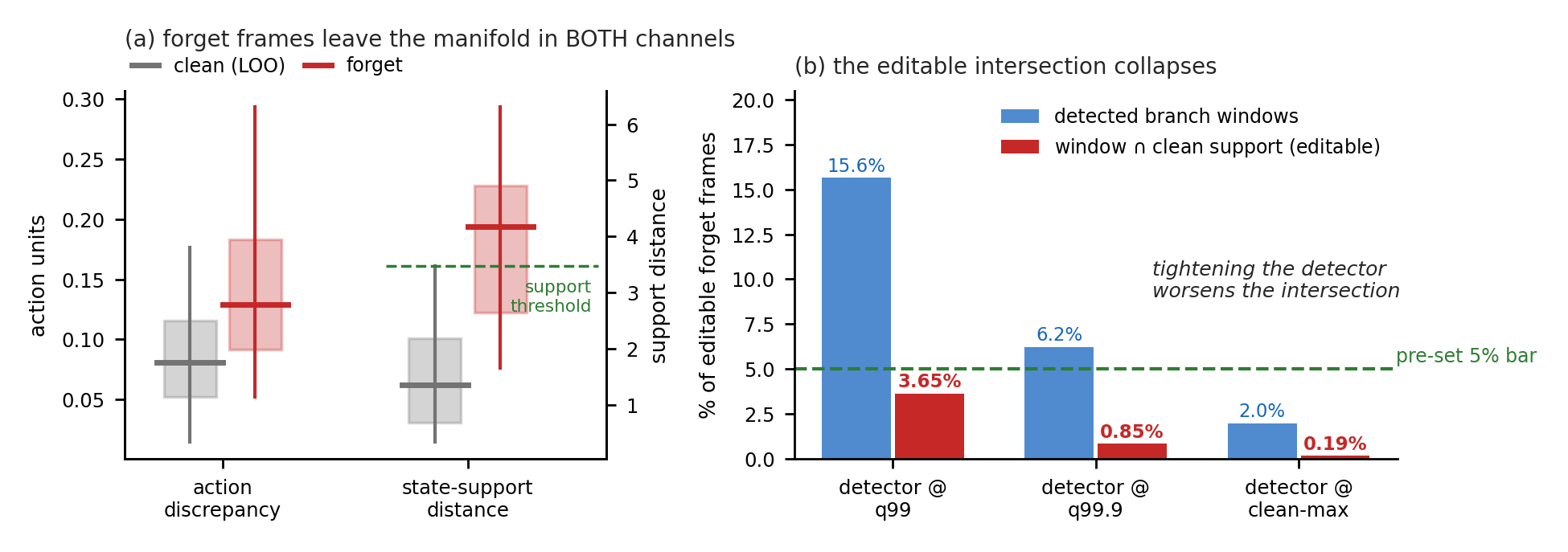}
\caption{\textbf{Why localized redirect has nothing to edit under
whole-episode contamination.} \emph{(a)} Quantile summaries (q05--q95
whiskers, q25--q75 boxes, q50 bars) of the two gating channels: forget frames
sit high in action discrepancy \emph{and} in state-support distance at once
--- state and action leave the clean distribution together. \emph{(b)}
Detected branch-window fraction vs.\ the support-gated editable intersection
at the three detector calibrations: 15.7$\to$3.65\%, 6.2$\to$0.85\%,
2.0$\to$0.19\% --- under the pre-set 5\% bar everywhere, and tightening the
detector monotonically worsens it. From the archived diagnostic JSON; no
thresholds recomputed.}
\label{fig:feasibility}
\end{widefigure}

\section{robomimic-BC: 14 methods, the fairness sweep, and the blindness
ledger}
\label{app:robomimic}
\label{sec:baselines}
Recovery is the fraction of the $\thetaz\!\to$retrain closed-loop success gap
recovered (negative = degradation below $\thetaz$); AUC-dist is the distance
of the membership AUC from the retrain null (0 = evidence at its target).
Redirect (there ``DIFT-pF'': $\lambda_{\text{prox}}{=}3$,
$\lambda_{\text{anchor}}{=}3$, calibrated stop) is the only method with
positive recovery; every alternative degrades closed-loop success by
68--101\% \emph{regardless of its offline forget quality}. Identification
(Stage 1, loss scorer): $P@60 = 1.000$ on constructed corr60 vs.\ $0.43$ on
natural worse-operator labels; on zero-mean action noise plain fine-tuning
wins outright, which scopes noise-type corruption out of the paper.

\begin{table}[ht]
\centering
\scriptsize
\setlength{\tabcolsep}{3pt}
\caption{\textbf{The 14-method comparison} (Can corr60, identical budgets,
500-rollout evals over gate-passing seeds).}
\label{tab:fourteen}
\begin{tabular}{@{}lccl@{}}
\toprule
method & recovery & AUC-dist & note \\
\midrule
redirect (ours) & $+12$ to $+33\%$ & 0.210 & only positive-recovery method \\
fine-tune $\retain$ & $-3.5\%$ & 0.249 & control; seed range $-36$ to $+49$ \\
EU-$k$ \citep{goel2022evaluating} & $-101\%$ & \textbf{0.008} & best evidence score; policy destroyed \\
Blindspot \citep{tarun2023deep} & $-97\%$ & 0.097 & prior regression-unlearning method \\
RMU & $-101\%$ & 0.109 & \\
NegGrad+ & $-101\%$ & 0.109 & \\
Fisher \citep{golatkar2020eternal} & $-101\%$ & 0.165 & \\
SimNPO & $-85\%$ & 0.156 & \\
NPO & $-68\%$ & 0.135 & closest mechanism sans calibration \\
SSD & $-98\%$ & 0.157 & \\
task-vector & $-94\%$ & 0.391 & moves \emph{away} from the null \\
GA / SCRUB \citep{kurmanji2023scrub} / relabel & $-46$ to $-130\%$ &
0.025--0.8 & earlier batch, same pattern \\
\bottomrule
\end{tabular}

\end{table}

\begin{table}[ht]
\centering
\footnotesize
\setlength{\tabcolsep}{4pt}
\caption{\textbf{The 32-configuration fairness sweep} (per-method
hyperparameter search plus our proximity tether transplanted into every
baseline; finalists at 500 rollouts, 2 seeds; saliency-masked unlearning
\citep{fan2024salun} was swept in the same grid). No fairly-tuned baseline
reaches positive recovery; the tether helps monotonically and rescues none;
the offline/closed-loop dissociation survives tuning (NegGrad+{+}tether:
AUC-dist 0.024, dead at $-63\%$).}
\label{tab:fairness}
\begin{tabular}{@{}lccc@{}}
\toprule
best-tuned baseline & mean SR & recovery & AUC-dist \\
\midrule
redirect (ref) & --- & $+12$ to $+33\%$ & 0.21 \\
NPO+tether(10) & 0.168 & $-31\%$ & 0.119 \\
NPO+tether(3) & 0.128 & $-41\%$ & 0.108 \\
SimNPO & 0.118 & $-53\%$ & 0.070 \\
NegGrad+{+}tether(10) & 0.068 & $-63\%$ & \textbf{0.024} \\
SSD & 0.070 & $-63\%$ & 0.218 \\
GA+tether(10) & 0.047 & $-70\%$ & 0.040 \\
Blindspot & 0.035 & $-77\%$ & 0.098 \\
\bottomrule
\end{tabular}

\end{table}

\begin{widetable}[ht]
\centering
\caption{\textbf{The blindness ledger}: standard audits disagreeing on one
checkpoint. Every cell is measured; rows 1--6 are the cross-class exhibits
of Sec.~\ref{sec:results}, rows 7--13 extend them. The PushT @0.50 row is
exploratory (Sec.~\ref{sec:generalize}).}
\label{tab:ledger}
\scriptsize
\setlength{\tabcolsep}{4pt}
\begin{tabular}{@{}p{0.16\textwidth}p{0.25\textwidth}p{0.27\textwidth}p{0.23\textwidth}@{}}
\toprule
checkpoint & audit A says & audit B says & mechanism \\
\midrule
ACT c50 redirect & behavior: 24.9\% of executed-slice gap closed & evidence:
AUC 1.000 $=\thetaz$ & masking, not deletion \\
ACT c50 ascent (primary lineage) & rank audit stops flagging by
200--400 steps & behavior never crosses; \texttt{mem/null} 1.6--1.8 &
rank acceptance through destructive overshoot \\
$\pi_{0.5}$ redirect (robot) & passes all three behavior gates; success
17/20 & per-demo attack: rank AUC 1.000 $=\thetaz$; \texttt{mem/null} 0.34 &
masking recurs cross-class (3-seed null) \\
ACT-c100 ascent & rank MIA at null (0.650 vs 0.653) & $+58\%$ past deployed
null & rank statistics blind to overshoot \\
DP ascent-200 & 32$\times$ loss-gap closure vs FT & zero action movement &
loss space $\neq$ behavior space \\
PushT ascent @0.50 (exploratory) & success $+8.0/{+}11.6/{+}2.4$ pts (3
seeds) & retain damaged $+34$--$69\%$; rank still flags all seeds
(0.96--0.99) & repair by asymmetric damage \\
\midrule
robomimic EU-$k$ & evidence at its target (AUC-dist 0.008) & closed-loop
$-101\%$ & the evidence axis cannot see lobotomy \\
robomimic SCRUB & near-perfect null-match (0.025) & closed-loop $-100\%$ &
forgetting vs.\ destruction, invisible to loss audits \\
robomimic NegGrad+{+}tether & AUC-dist 0.024 after fair tuning & $-63\%$
closed loop & tuning does not repair the dissociation \\
robomimic task-vector & AUC-dist 0.391, past the null & $-94\%$ closed loop &
negative erasure: audit flags editing, not cleanliness \\
DP FT500 & closes 5.9\% of the loss gap (vs.\ 0.19\% at 200) & zero action
movement & FT ``forgetting'' grows with budget \\
PushT ascent @0.10 & forget loss $+7.8\%$ & behavior $-2.4$ pts; rank moves
$-0.2\%$ & loss space moves before any audit does \\
ACT ascent, seeds b/c & $65$--$74\%$ of forget-loss gap closed, full
budget & rank still flags (0.86--0.91); absolute at/near null; BRANCH
conduct moves \emph{toward} the retrain (B300, highest banked rung) &
the training draw, not the budget, sets both endpoints \\
\bottomrule
\end{tabular}
\end{widetable}

\section{Hardware appendix}
\label{app:hw}
\emph{Film-strip (fixed rule).} Fig.~\ref{fig:filmstrip}: frames at
identical fractions of each episode's own duration plus the recorded terminal
frame, trial 0 (the same initial position in every condition by the paired
protocol); no frame chosen by eye.

\begin{widetable}[ht]
\centering
\caption{\textbf{Robot trials by condition, blind-scored} (200 anonymized
shuffled clips, sealed mapping opened only after scoring): success with
Wilson 95\% CIs, exact McNemar vs.\ each arm's ceiling paired by initial
position, and the dominant failure mode. Live keys agreed with the blind
labels on 199/200 episodes (blind authoritative). DP hardware is excluded
from this version.}
\label{tab:trials}
\footnotesize
\begin{tabular}{@{}llcccl@{}}
\toprule
arm & condition & success & Wilson 95\% CI & McNemar vs.\ ceiling & dominant failure mode \\
\midrule
ACT & ceiling & 20/20 & [0.84, 1.00] & --- & --- \\
ACT & redirect & 18/20 & [0.70, 0.97] & $b{=}2$, $c{=}0$, $p{=}0.500$ & dropped\_short (1) \\
ACT & FT & 9/20 & [0.26, 0.66] & $b{=}11$, $c{=}0$, $p{<}0.001$ & sway\_and\_knock (8) \\
ACT & ascent & 7/20 & [0.18, 0.57] & $b{=}13$, $c{=}0$, $p{<}0.001$ & sway\_and\_knock (10) \\
ACT & $\thetaz$ & 5/20 & [0.11, 0.47] & $b{=}15$, $c{=}0$, $p{<}0.001$ & sway\_and\_knock (14) \\
\addlinespace[2pt]
$\pi_{0.5}$ & ceiling & 19/20 & [0.76, 0.99] & --- & dropped\_short (1) \\
$\pi_{0.5}$ & redirect & 17/20 & [0.64, 0.95] & $b{=}3$, $c{=}1$, $p{=}0.625$ & dropped\_short (3) \\
$\pi_{0.5}$ & FT & 18/20 & [0.70, 0.97] & $b{=}2$, $c{=}1$, $p{=}1.000$ & dropped\_short (1) \\
$\pi_{0.5}$ & ascent & 8/20 & [0.22, 0.61] & $b{=}12$, $c{=}1$, $p{=}0.003$ & sway\_and\_knock (9) \\
$\pi_{0.5}$ & $\thetaz$ & 2/20 & [0.03, 0.30] & $b{=}17$, $c{=}0$, $p{<}0.001$ & sway\_and\_knock (17) \\
\addlinespace[2pt]
\bottomrule
\end{tabular}

\end{widetable}

\begin{widefigure}[ht]
\centering
\includegraphics[width=0.8\linewidth]{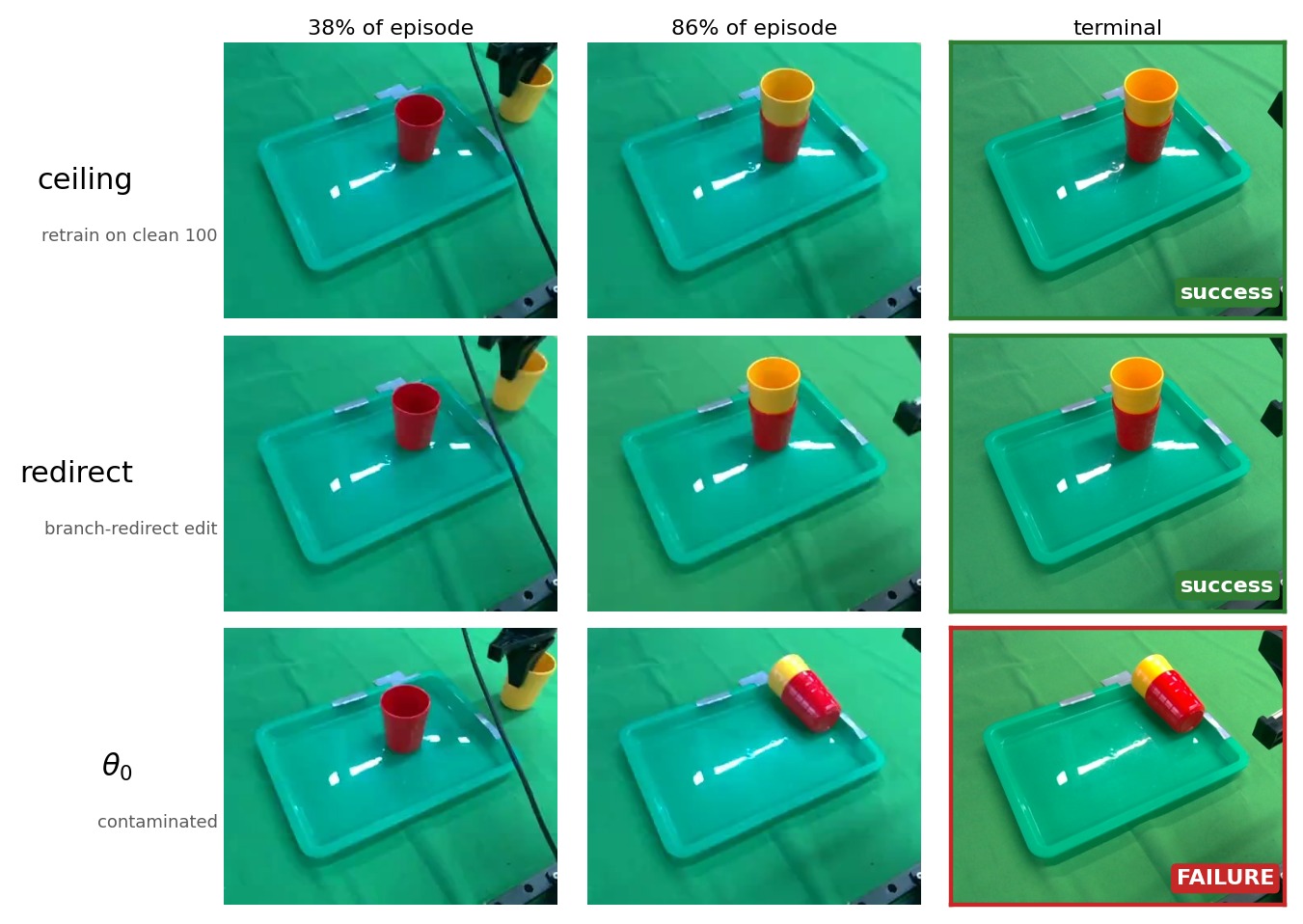}
\caption{\textbf{$\pi_{0.5}$ hardware film-strip}, position-matched by the
paired protocol (trial 0): the retrain seats the cup, the redirect edit
matches it, $\thetaz$ topples the assembly. Frames at fixed fractions of each
episode's duration; terminal frame recorded by the rig.}
\label{fig:filmstrip}
\end{widefigure}

\label{app:gaps}
\emph{Per-arm gaps at this version.}
ACT: all hardware outcomes are blind-scored (Table~\ref{tab:trials}); the
$\thetaz$ valid-20 rule (Sec.~\ref{sec:dissoc}) was fixed before scoring.
$\pi_{0.5}$: the per-demonstration MIA is measured (rank 1.000 on $\thetaz$
and every edit; \texttt{mem/null} 0.10--0.44) against a 3-seed LoRA
re-fit null ($0.558 \pm 0.012$); the arm's behavioral evaluation remains
one task family at $n{=}20$. DP: offline cell only in this version; its
$\thetaz\!\to$floor BRANCH gap sits within the floor's own pair spread, so
closure fractions are undefined on that arm. PushT: the high-dose
\emph{ladder} is dose-resolved on s1000 only, but its r050 endpoint is
measured on all three seeds; non-members are synthesized mirrors
(validated by the null-construction check). robomimic: fairness finalists
are on 2 seeds. ACT: evidence statistics carry the frame-draw noise
decomposition of Sec.~\ref{sec:offline}.

\begin{widetable}[ht]
\centering
\caption{\textbf{The two-axis matrix}, one block per policy class, each
block read by its own pre-declared instruments
(Table~\ref{tab:conventions}): behavior --- \% of the
$\thetaz\!\to$reference gap closed, on the counterfactual-conduct axis by
offline action gap (ACT) and on the closed-loop \emph{manifestation}
channel by success (PushT; $\pi_{0.5}$ blind-scored). The channels
dissociate on the same checkpoint ($\pi_{0.5}$ FT, Sec.~\ref{sec:hw}),
so a manifestation cell never stands in for a conduct claim; evidence ---
rank AUC vs.\ that arm's pooled retrain null, difference from the null, and
the absolute \texttt{mem/null} level. A qualitative cross-arm ledger, not a
quantitatively uniform comparison: cells compare within a block, never
across blocks. Auto-generated from the measured artifacts by
\texttt{table1\_matrix.py}; a cell without a measurement would print a red
placeholder rather than a number. $\pi_{0.5}$ evidence cells were measured
2026-08-17 (per-episode artifacts in the release archive); DP has no closure
instrument, so its block is absent (Appendix~\ref{app:gaps}).}
\label{tab:matrix}
\footnotesize
\begin{tabular}{@{}lcccc@{}}
\toprule
& \multicolumn{1}{c}{behavior} & \multicolumn{3}{c}{evidence} \\
\cmidrule(lr){2-2}\cmidrule(lr){3-5}
& \% of gap closed & rank AUC & vs.\ null & \texttt{mem/null} \\
\midrule
\multicolumn{5}{@{}p{\dimexpr\linewidth-2\tabcolsep}@{}}{\textbf{ACT (chunk-50)} \footnotesize primary $\thetaz$ lineage (seed replication: Sec.~5.1); behavior = executed-slice BRANCH closure vs.\ 3-retrain floor (conduct axis; full-plan: App.~\ref{app:floors}); retrain null AUC $0.639$}\\[1pt]
\quad $\thetaz$ & +0.0 & 1.000 & +0.36 & 0.22 \\
\quad redirect R200 & +24.9 & 1.000 & +0.36 & 0.40 \\
\quad redirect R400 & +23.4 & 0.990 & +0.35 & 0.44 \\
\quad ascent B200 & -57.8 & 0.650 & +0.01 & 0.84 \\
\quad ascent B400 & -182.3 & 0.533 & -0.11 & 1.83 \\
\quad FT200 & +3.4 & 1.000 & +0.36 & 0.19 \\
\addlinespace[3pt]
\multicolumn{5}{@{}p{\dimexpr\linewidth-2\tabcolsep}@{}}{\textbf{diffusion-PushT (s1000)} \footnotesize behavior = success vs.\ 5-retrain fleet, 250 rollouts (manifestation channel, not matched-state conduct); retrain null AUC $0.534$}\\[1pt]
\quad $\thetaz$ & +0.0 & 0.981 & +0.45 & 0.11 \\
\quad ascent @0.20 (frozen) & -2.0 & 0.979 & +0.45 & 0.15 \\
\quad ascent @0.50 (exploratory) & +39.5 & 0.964 & +0.43 & 0.27 \\
\quad FT200 & +5.9 & 0.980 & +0.45 & 0.11 \\
\addlinespace[3pt]
\multicolumn{5}{@{}p{\dimexpr\linewidth-2\tabcolsep}@{}}{\textbf{$\pi_{0.5}$ (robot checkpoints)} \footnotesize behavior = blind-scored success closure vs.\ ceiling ($n{=}20$ paired; manifestation channel --- offline conduct: Sec.~5.2); \texttt{mem/null} = within-model ratio vs.\ the 3-seed ceiling pool; retrain null AUC $0.558$}\\[1pt]
\quad $\thetaz$ & +0.0 & 1.000 & +0.44 & 0.10 \\
\quad redirect (robot) & +88.2 & 1.000 & +0.44 & 0.34 \\
\quad ascent (robot) & +35.3 & 1.000 & +0.44 & 0.44 \\
\quad FT (robot) & +94.1 & 1.000 & +0.44 & 0.34 \\
\bottomrule
\end{tabular}

\end{widetable}

\begin{figure}[ht]
\centering
\includegraphics[width=\linewidth]{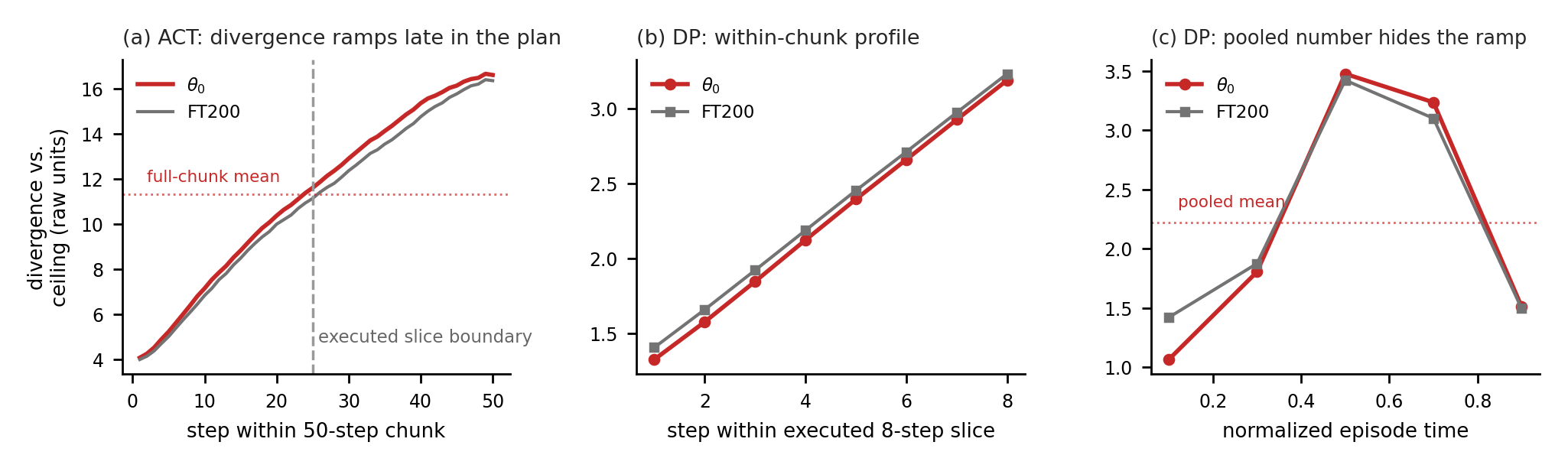}
\caption{\textbf{Where divergence lives in time --- pooled magnitudes hide
it.} \emph{(a)} ACT $\thetaz$-vs-ceiling divergence by step within the
50-step chunk (pooled forget states): the profile ramps $4\to17$ raw units
across the plan, which is why the full-chunk mean and the executed 25-step
mean support different closure fractions (Sec.~\ref{sec:offline}) while the
operator ordering is unchanged. \emph{(b)} DP within its executed 8-step
slice. \emph{(c)} The same DP divergence re-binned over normalized episode
time: it concentrates mid-episode at $1.6\times$ the pooled mean. The FT
control tracks $\thetaz$ almost exactly in every panel --- budget-matched
fine-tuning moves essentially nothing. From the banked temporal profiles
(\texttt{act\_floor\_c50\_raw.json}, \texttt{dp\_floor\_report.json}).}
\label{fig:manifest}
\end{figure}

\clearpage
\section{Rank-audit invariance}
\label{app:prop}
\begin{proposition}[Rank audits are blind to absolute overshoot]
\label{prop:rank}
Suppose an auditor's verdict depends on the scores $\{s(x_i)\}$ of the
audit pool only through their ordering, as any ROC/AUC statistic does. Then
for every strictly increasing $g$ the scores $g \circ s$ produce the same
verdict, so two edited models whose score orderings agree receive identical
rank verdicts even when their absolute member levels differ arbitrarily. In
particular ``member loss at the retrain null'' and ``member loss past it''
(outcome 3) can be rank-indistinguishable; absolute calibration is a
separate, necessary check.
\end{proposition}
\emph{Proof.} $\mathrm{AUC} = \Pr[s(m) > s(n)] + \tfrac{1}{2}\Pr[s(m) =
s(n)]$ over member--non-member pairs. A strictly increasing $g$ preserves
every pairwise comparison, hence every rank functional; location on the
score axis is unrecoverable from ranks alone. \qed

Table~\ref{tab:evladder} realizes it in data: B200 and B300 sit on
opposite sides of the null in absolute level (\texttt{mem/null} 0.84 vs.\
1.59) while both rank CIs contain the null.

\section{The joint test, and where evidence lives}
\label{app:joint}

\emph{Joint conformal test.} The criterion of Sec.~\ref{sec:outcomes},
instantiated on ACT: coordinates are executed-slice BRANCH divergence to
the retrain fleet (raw units), $|\log \texttt{mem/null}|$ against the
5-seed pooled member level, and $|\mathrm{AUC} - 0.639|$. The fleet is
$K{=}19$ retrains of the identical recipe on the identical 100-episode
retain set, differing only in seed, and it nests the other instruments:
its first 3 seeds form the behavioral floor's pairs
(Sec.~\ref{sec:design}), its first 5 the evidence null, and all 19 the
conformal calibration set. Each replica is scored leave-one-out;
coordinates are standardized by the replicas' spread and reduced by
sup-norm. Every checkpoint is rejected at $p{=}0.050$ --- the floor
$1/20$ of a 19-replica fleet, and the attainable minimum because each
checkpoint is more atypical than every replica. Growing the fleet from
5 to 19 widened the calibration envelope (leave-one-out nonconformity
$0.38$--$3.32$, median $1.04$) while the rejection sharpened from
floor-limited $p{=}1/6$ to conventional significance: checkpoint
nonconformity is $16$--$57$, i.e.\ $4.9$--$17.2\times$ the most extreme
replica. The nonconformity column is reported for transparency,
\emph{not} as a calibrated ranking of failures.

\begin{table}[ht]
\centering
\caption{\textbf{Joint retrain-consistency, one test per checkpoint,
$K{=}19$ retrain replicas.} The replicas' own leave-one-out
nonconformity spans $0.38$--$3.32$; every checkpoint scores $16$--$57$
and is rejected at the fleet's attainable minimum $p{=}0.050$. Redirect
is nearest; ascent's high rungs are farthest --- rank acceptance and
joint rejection disagree on the same checkpoints.}
\label{tab:conformal}
\footnotesize
\setlength{\tabcolsep}{4pt}
\begin{tabular}{@{}lccccc@{}}
\toprule
 & \multicolumn{3}{c}{audit vector $z$} & & \\
\cmidrule(lr){2-4}
 & BRANCH div. & $|\log$ \texttt{m/n}$|$ & $|\Delta$AUC$|$ & $s$ & conformal $p$ \\
\midrule
$\thetaz$ & 4.00 & 1.54 & 0.34 & 31 & 0.050 \\
FT200 & 3.92 & 1.69 & 0.34 & 33 & 0.050 \\
redirect R200 & 3.47 & 0.92 & 0.34 & 18 & 0.050 \\
redirect R400 & 3.47 & 0.83 & 0.33 & 16 & 0.050 \\
ascent B50 & 3.97 & 0.97 & 0.33 & 21 & 0.050 \\
ascent B100 & 4.30 & 0.58 & 0.28 & 24 & 0.050 \\
ascent B200 & 5.25 & 0.18 & 0.01 & 32 & 0.050 \\
ascent B300 & 7.99 & 0.45 & 0.11 & 57 & 0.050 \\
ascent B400 & 7.95 & 0.60 & 0.13 & 57 & 0.050 \\
\bottomrule
\end{tabular}

\end{table}

\emph{Where evidence lives.} Region-resolved membership evidence: every
frame of all 30 member episodes scored under the same instrument as the
published audit, split by the frozen region labels, as \texttt{mem/null}
against the pooled per-region ceiling level. Predictions frozen before
scoring: redirect moves BRANCH strongly and non-BRANCH barely; ascent
moves every region; FT moves none. All three confirmed
(Table~\ref{tab:regions}).

\begin{table}[ht]
\centering
\caption{\textbf{Region-resolved \texttt{mem/null}.} Redirect reaches
the retrain null on the $8.6\%$ of frames it edits and nowhere else ---
the episode-mean statistic dilutes a real, local repair $12{:}1$. Ascent
(the overshooting lineage) overshoots every region; FT deepens
memorization slightly everywhere.}
\label{tab:regions}
\footnotesize
\setlength{\tabcolsep}{5pt}
\begin{tabular}{@{}lcccc@{}}
\toprule
region (frame share) & $\thetaz$ & FT200 & redirect R200 & ascent B300 \\
\midrule
PRE (46.6\%) & 0.33 & 0.32 & 0.52 & 2.31 \\
\textbf{BRANCH} (8.6\%) & 0.26 & 0.21 & 1.07 & 2.19 \\
POST (7.8\%) & 0.20 & 0.17 & 0.68 & 1.27 \\
TAIL (37.0\%) & 0.17 & 0.13 & 0.20 & 1.30 \\
\bottomrule
\end{tabular}

\end{table}

\emph{The ascent endpoint is a stopping-rule artifact.} The operator's
forget term raises member losses only toward $\tau$, a running quantile
of the policy's \emph{own} loss distribution, recomputed every 50 steps.
Across the three $\thetaz$ lineages (which start indistinguishable:
AUC $1.000$, \texttt{mem/null} $0.19$--$0.22$), $\tau$ saturates by
step 300 at $0.547/0.362/0.413$ against a null member level of $0.461$;
every run reached full budget with no guard trips. Overshoot occurs
exactly when the saturated $\tau$ passes the null --- one lineage
overshoots to $1.83$, one lands \emph{at} the null ($0.91$) with its
rank audit still flagging ($0.913$ vs.\ $0.639$; the mirror image of
the primary lineage's rank-accepted overshoot), one sits between
($1.19$). Banked conduct splits the same way: the primary lineage's
B300 moves BRANCH divergence away from the retrain ($-128\%$ of the
$\thetaz$ gap), the other two move toward it ($+63\%$, $+27\%$).
Loss-space movement again proxies nothing: $65$--$74\%$ of
the forget-loss gap closed buys $\le 0.14$ of rank AUC.

\emph{Parameter space.} At matched BRANCH states, the cosine between the
membership loss gradient and the counterfactual-proximity gradient is
$+0.66$ at $\thetaz$, $+0.30$ at the repaired R200 --- the two audits
constrain nearly orthogonal directions at exactly the checkpoint where a
localized repair acted. (At B300 the cosine reads $+0.98$, but both losses
are far from both targets there, so the sign structure agrees mechanically;
that reading carries no information and is excluded from the claim.)

\section{Pre-registered predictions and outcomes (PushT arm)}
\label{app:prereg}
Frozen in the runbook before the first training started
(\texttt{PUSHT\_RUNBOOK\_5090B.md}; ``do not edit after the first training
starts''), reproduced verbatim, with the measured outcome after each.

\begin{itemize}
\item \textbf{P1 (expression).} \emph{``diffusion is mode-preserving, so
theta\_0 expresses the mirrored mode from nominal starts at a nonzero rate
(rollouts ending nearer the mirrored goal than the true goal), and its
success rate sits $\ge$10 pts below the ceiling with non-overlapping 95\%
Wilson CIs (this is the GATE).''} --- \textbf{Confirmed} (gate passed
2026-08-14): gap 14.8 points, seed-level Welch $t{=}3.78$, $p\approx0.018$,
every $\thetaz$ seed below every ceiling seed
(Table~\ref{tab:pushtgate}); \texttt{avg\_max\_reward} separates in the same
direction.
\item \textbf{P2 (redirect).} \emph{``branch-redirect recovers success toward
the ceiling; the recovered fraction is resolvable at n=250 (CI half-width
$\sim$$<$6 pts).''} --- \textbf{Prospectively not applicable}: the frozen
feasibility diagnostic (Appendix~\ref{app:feas}) finds a 3.65\% editable
intersection against the pre-set 5\% bar, so no redirect edit ran; P2 was
not evaluated by silently substituting an all-frames operator (dated
amendment, 2026-08-14).
\item \textbf{P3 (ascent).} \emph{``delivered ascent (norm-PCGrad) erases
membership evidence with budget but moves behavior away from the retrain
counterfactual --- the c50 signature, replicated prospectively.''} ---
\textbf{Refuted at the frozen dose} on all three $\thetaz$ seeds: forget
loss moves $+40$--$78\%$ while rank AUC moves $\le0.003$, the absolute gap
$4$--$6\%$, and behavior $+1.6$ points pooled ($t{=}1.39$, inside the null).
By the runbook's own rule, a failed prediction is a finding
(Sec.~\ref{sec:generalize}).
\item \textbf{P4 (FT control).} \emph{``budget-matched fine-tuning on retain
moves neither axis.''} --- \textbf{Confirmed} (rank $-0.001$ to $-0.003$ per
seed; behavior $-2.1$ points pooled, $t{=}{-}0.95$).
\item \textbf{P5 (evidence).} \emph{``theta\_0 MIA AUC $\gg$ retrain null;
redirect stays at theta\_0's level (behavioral masking); the null comes from
$\ge$3 ceiling seeds.''} --- \textbf{Confirmed for $\thetaz$} ($0.981$--$0.997$
vs.\ a 5-seed null of $0.534\pm0.012$); the redirect clause was not evaluated
(see P2).
\end{itemize}

The post-gate operator-instantiation amendment (2026-08-14, recorded before
any redirect edit ran) is preserved verbatim in the runbook alongside the
frozen feasibility diagnostic's archived output.

\end{document}